\documentclass[11pt]{article}

\usepackage[preprint]{acl}

\usepackage{times}
\usepackage{latexsym}

\usepackage[T1]{fontenc}

\usepackage[utf8]{inputenc}

\usepackage{microtype}

\usepackage{inconsolata}

\usepackage{bm}
\usepackage{amsmath, amssymb, amsfonts, amsthm}
\usepackage{graphicx}
\usepackage{booktabs}
\usepackage{multirow}
\usepackage{amsfonts}
\usepackage{color,xcolor,colortbl}
\usepackage{pifont}
\usepackage{adjustbox}
\usepackage{tabularx}
\usepackage{subcaption}
\usepackage{cleveref}
\usepackage{makecell}
\usepackage{enumitem}
\usepackage{mathtools}
\usepackage{hyperref}
\usepackage{paralist}
\usepackage{thmtools, thm-restate}
\usepackage{array}

\declaretheorem{lemma, proposition, corollary}[
    style=plain
]

\definecolor{oursgray}{gray}{0.95}
\definecolor{darkgreen}{RGB}{0,130,0}
\definecolor{darkred}{RGB}{160,0,0}
\newcommand{\mse}[2]{\ensuremath{#1{\scriptstyle\,\pm\,#2}}}
\newcommand{\bestmse}[2]{%
  \ensuremath{\mathbf{#1}{\scriptstyle\,\mathbf{\pm\,#2}}}}
\newcommand{\secondmse}[2]{\underline{\mse{#1}{#2}}}
\newcommand{\checkmk}{\textcolor{darkgreen}{\ding{51}}}
\newcommand{\cross}{\textcolor{darkred}{\ding{55}}}
\newcommand{\rev}[1]{\textcolor{black}{#1}}

\title{Truth in Contrast: Calibrated Factual Evidence from Language Model Likelihoods}

\author{
    Ponhvoan Srey\textsuperscript{1} \quad
    Jason Lucas\textsuperscript{2} \quad
    Pittawat Taveekitworachai\textsuperscript{1} \\
    \textbf{Junyoung Lee\textsuperscript{3}} \quad
    \textbf{Anh Tuan Luu\textsuperscript{1,4}\thanks{Corresponding Author}}\\
    \textsuperscript{1}Nanyang Technological University \quad
    \textsuperscript{2}University of Colorado at Boulder\\
    \textsuperscript{3}National University of Singapore\quad
    \textsuperscript{4}VinUniversity\\
    \texttt{\{ponhvoan002,anhtuan.luu\}@ntu.edu.sg}
}

\begin{document}
\maketitle
\begin{abstract}

Large Language Models (LLMs) encode substantial factual knowledge. Yet, estimating confidence in factual claims reliably remains challenging.
We study likelihood-based factual confidence from the complementary perspectives of discrimination and calibration across multilingual factual statements.
Specifically, we instantiate contrastive generative likelihoods as Contrastive Factual Evidence (CFE), which compares the likelihood of the same claim under competing factual and non-factual conditions, and derive its probabilistic counterpart, the Contrastive Factual Posterior (CFP).
We show analytically that CFP calibration depends on the differential error between the two surrogate conditionals, allowing errors shared across them to cancel and yield tighter calibration bounds.
Empirically, CFP substantially improves calibration while maintaining factual discrimination comparable to its strongest competitor.
\rev{Nonetheless, discrimination is uniformly weak across methods in lower-resource languages, which simultaneously undermines the utility of the measured calibration.}
To this end, we interpret a likelihood geometry, within which an evidence-bearing direction is induced, along with an orthogonal evidence-invariant direction.
In our investigation, we selectively align target-language claims with stronger English evidence only along the factual-evidence direction, improving discrimination and open-ended factual generation without target-language labels.

\end{abstract}

\section{Introduction}

Large Language Models (LLMs) acquire substantial factual knowledge during training, yet they can still produce incorrect information \citep{guan-etal-2024-language, jiang-etal-2024-large}.
As a result, a fundamental goal has been to expose this underlying truth-related information to enable LLMs to distinguish factual from non-factual claims.
Extensive work shows that such information can be recovered from model internal representations \citep{azaria-mitchell-2023-internal, marks2024the}.
A complementary, albeit arguably less explored, approach is to derive factuality signals from the model's own likelihoods, offering a training-free and probabilistically interpretable measure of factuality.
Likelihood-based \textbf{factual confidence} captures this objective by assigning a claim a continuous score that reflects an LM's confidence in its correctness \citep{mahaut-etal-2024-factual}.
Unlike a binary factuality decision, this confidence can express degrees of reliability and thereby support downstream mechanisms such as verification, filtering, abstention, or further review \citep{geng-etal-2024-survey}.

\begin{figure}
    \centering
    \includegraphics[width=\linewidth]{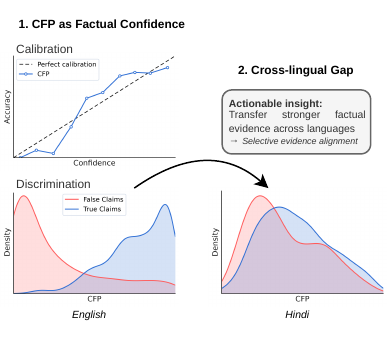}
    \caption{Overview.
    (1) We instantiate CFP as factual confidence and evaluate through complementary calibration and discrimination lenses (\S~\ref{sec:expt}). (2) Discrimination degrades in lower-resource languages \rev{despite low ECE}, motivating selective evidence transfer from stronger languages (\S~\ref{sec:fact_align}).}
    \label{fig:overview}
\end{figure}

\rev{
For factual confidence to be useful, it should be both \textbf{discriminative} and \textbf{calibrated}.
Discrimination reflects whether confidence can rank factual above non-factual claims, which is especially relevant to ranking- or selection-based uses such as filtering and selective prediction;
calibration instead determines whether the numerical confidence can be interpreted as a probability of correctness \citep{geng-etal-2024-survey,pmlr-v70-guo17a}.
These properties need not coincide.
Moreover, both should remain reliable under semantically equivalent variation, yet existing confidence estimators can degrade across languages and input forms \citep{mahaut-etal-2024-factual, xue-etal-2025-mlingconf}.}
This raises our central question: \emph{what signal can provide confidence that is both discriminative and calibrated while remaining robust to irrelevant linguistic variation?}

Common estimators either derive confidence from predictive uncertainty \citep{huang2025look}, which need not track factual correctness \citep{srey2026towards}, or directly elicit truth judgments such as P(True) \citep{kadavath2022languagemodelsmostlyknow}. %
While the latter more directly targets factual correctness, it relies on the model successfully expressing its knowledge through a prescribed truth-judgement instruction interface, whose behaviour can depend on prompting and post-training.
Neither perspective is explicitly designed to isolate factuality-specific evidence from unrelated variation.

We take a complementary \emph{generative view} grounded in Bayesian inference.
Rather than asking whether a claim is true, we ask how likely the same claim would be under factual versus non-factual generative conditions.
Their relative likelihood provides \textbf{Contrastive Factual Evidence (CFE)} and induces a probabilistic counterpart, \textbf{Contrastive Factual Posterior (CFP)}.
At the same time, because both likelihoods score the same claim, components of surrogate error shared across the two conditions can cancel, leaving the contrast dependent only on their differential error and providing a mechanism through which calibration can improve.
Then, we interpret the paired conditional likelihoods underlying CFE geometrically as a two-dimensional likelihood space, where one direction determines factual evidence while the orthogonal direction is evidence-invariant.
As a proof of concept, we exploit this decomposition to motivate an evidence-selective cross-lingual alignment objective designed to transfer stronger English factual evidence to target languages without requiring target-language factuality labels.
While the underlying generative likelihood contrast is well established, we study its implications specifically for factual confidence, calibration, and cross-lingual transfer.

\Cref{fig:overview} provides an overview of this work.
In summary, our main contributions are threefold:
\begin{inparaenum}[(i)]
\item We study contrastive generative likelihoods as factual confidence through CFE and CFP, \rev{jointly examining their calibration and discrimination across multilingual settings and exposing substantial cross-lingual variation between these two properties.} %
\item We characterise CFP calibration error, formalising our intuition on how calibration can be improved.
\item We derive an evidence-selective likelihood geometry and highlight the potential for cross-lingual factuality alignment.
\end{inparaenum}

\section{Related Work}

\paragraph{Generative Classification.}

Generative classifiers infer class labels via Bayes' rule from class-conditional likelihoods, rather than modelling posterior class probabilities directly \citep{ng2001discriminative}, and have well-established precedents in natural language processing. Early work includes utilising generative models for text classification \citep{mccallum1998comparison, nigam2000text} and readability assessment \citep{schwarm-ostendorf-2005-reading}.
This perspective has been revisited with pretrained LLMs, including zero-shot text classification using natural language label descriptions \citep{kumar2024gen}, few-shot classification with noisy-channel prompting \citep{min-etal-2022-noisy}, generative calibration for in-context learning \citep{jiang-etal-2023-generative}, and class-conditioned decoding \citep{krause-etal-2021-gedi-generative}.
Generally, these methods condition generation on a candidate attribute or class and compare how well the observed text is explained under each alternative.
\citet{bigelow2026belief} provide a Bayesian interpretation of LM behaviour, viewing likelihood changes as accumulated evidence about latent concepts. Building on this broader evidence perspective, \citet{hu-etal-2026-string} study a closely related principle of contrastive isolation.
Raw string probability reflects grammaticality, their desired attribute, alongside other factors, and controlled contrasts can better expose this attribute.
Our motivation is similar but we instantiate this contrast differently: we hold the claim fixed and compare its likelihood under factual and non-factual generative conditions to isolate evidence associated with factuality.
Moreover, where prior generative classification work primarily targets categorical prediction, we focus on the induced likelihood ratio as a continuous confidence signal.

\paragraph{Uncertainty and Confidence.}
A broad family of estimators derives confidence directly from the model's predictive distribution, using quantities such as sequence log-likelihood, perplexity or surprisal, predictive entropy, and energy-based scores \citep{huang2025look, liu2020energy}. %
Meanwhile, direct elicitation explicitly asks the model to evaluate the factuality of a statement or correctness of a proposed answer. For example, P(True) \citep{kadavath2022languagemodelsmostlyknow} uses the probability assigned to a ``True'' judgement as confidence.
Subsequent work \citet{heo2025uncertainty} has also considered a contrastive variant, which normalises ``True'' against ``False'' probabilities, while \citet{lemesle-etal-2025-paraphrase} similarly use a log-likelihood ratio between opposing judgments. These methods contrast judgments conditioned on the input, whereas CFE reverses the conditioning direction following Bayes' perspective.
Recent evaluative frameworks show that factual confidence can be sensitive to prompt \citep{mahaut-etal-2024-factual} and cross-lingual variation \citep{xue-etal-2025-mlingconf}.
Furthermore, confidence can also be elicited through verbalisation \citep{tian-etal-2023-just, xiong2024can}. We focus instead on likelihood-based confidence that does not require an additional verbalised self-assessment.

\paragraph{Truth Representation and Alignment.} A complementary line of work studies how factuality is encoded in LLM internal representations. Prior work has shown that truthfulness can be decoded from hidden activations \citep{azaria-mitchell-2023-internal, burns2023discovering} and is represented approximately linearly along truth-related directions \citep{marks2024the}.
These findings have facilitated downstream hallucination detection work \citep{ch-wang-etal-2024-androids, srey-etal-2025-unsupervised} and are closely linked to interventions that steer internal representations toward greater truthfulness \citep{li2023inferencetime, zhang-etal-2024-truthx}.
Related work on multilingual factual knowledge has identified language-dependent components in factual recall \citep{fierro-etal-2025-multilingual} and substantial variation in factual retrieval across languages \citep{jiang-etal-2020-x}.
This cross-lingual disparity has motivated methods that improve factual consistency across languages \citep{wang-etal-2025-lost-multilinguality, liu2026posttraining}.
Our likelihood geometry builds on these perspectives by deriving a truth-relevant direction for cross-lingual transfer, but differs in construction.
Rather than learning a truth direction from high-dimensional hidden states, we isolate a factual-evidence direction induced by the conditional likelihoods and align only this component across languages.

\section{Factuality as Contrastive Generative Evidence}

In this section, we formulate Contrastive Factual Evidence (CFE), by contrasting LM class-conditional likelihoods, and its corresponding probabilistic confidence, Contrastive Factual Posterior (CFP).
We then analyse how shared error between the surrogate conditionals can cancel in this contrast, clarifying when such cancellation can lead to improvement in calibration of factual confidence.

\subsection{Problem Setup}

Let $Y$ denote a natural-language claim and $C\in\{0,1\}$ its factuality, where $C=1$ indicates that the claim is factual. We write $p^*_c(y)=P(Y=y \mid C=c)$ for the unknown true class-conditional distribution and $\pi_c = P(C=c)$ for the corresponding class prior.

Our goal is to estimate the factual posterior 
\begin{equation}
    \eta (y) = P(C=1\mid Y=y)
\end{equation}
from a language model $q_\theta$. For convenience, we use $z_1$ and $z_0$ to denote natural-language conditions describing factual and non-factual statements, respectively, and define the induced surrogate distributions $q_c(y) = q_\theta (y\mid z_c)$.

\subsection{Factual Evidence and Confidence from Conditional Likelihoods}

By Bayes' rule,

\begin{equation}
\begin{split}
\eta(y) &= P(C=1\mid Y=y) \\ 
&= \frac{P(Y=y\mid C=1)P(C=1)}{\sum_{c\in\{0,1\}} P(Y=y\mid C=c)P(C=c)} \\
&= \frac{1}{1 + \frac{P(Y=y\mid C=0)P(C=0)}{P(Y=y\mid C=1)P(C=1)}} \\
&= \sigma \left( \log \frac{p^*_1(y)}{p^*_0(y)} + \log \frac{\pi_1}{\pi_0}\right),
\end{split}
\end{equation}
where $\sigma$ is the sigmoid function. We denote $E^*(y)=\log \left(p^*_1(y)/p^*_0(y)\right)$ as the ideal log-evidence and $\alpha=\log(\pi_1/\pi_0)$ as prior log-odds.

In practice, the true conditionals $p^*_c$ are unavailable. We instead use an LM to obtain surrogate conditionals $q_c$, giving the \textbf{Contrastive Factual Evidence (CFE)}
\begin{equation}
    \hat{E}(y) = \log q_1(y) - \log q_0(y)
\end{equation}
and the probabilistic confidence, \textbf{Contrastive Factual Posterior (CFP)}
\begin{equation}
    \hat\eta(y) = \sigma \left(\hat{E}(y) + \hat{\alpha}\right)
\end{equation}

Unless otherwise specified, these scores are normalised by token length, and are thus approximations of the true posterior, and are averaged across five pairs of paraphrased conditions. We provide detailed formulation in \Cref{subsec:cfe_formulation}. \Cref{tab:ablation} outlines ablation results over several formulation choices.

\subsection{Calibration under Shared Surrogate Error}

Define the surrogate log-error for each $c$
\begin{equation}
    \epsilon_c (y) = \log q_c(y) - \log p^*_c(y) 
\end{equation}
and prior misspecification $\delta_\pi = \hat{\alpha} - \alpha$. Then, the total posterior-logit error is $\epsilon_1 - \epsilon_0 + \delta_\pi$.
CFE therefore depends only on the differential, rather than absolute, surrogate error. Consequently, error components shared across the two conditionals are suppressed by subtraction. %
Our following analysis reveals that this cancellation translates into improved calibration when the two errors are correlated.

\begin{restatable}[{\citealp{kumar2019verified}}]{definition}{calibration}
    For a probabilistic predictor $f(Y)\in [0,1],$ the calibration error is
    \begin{equation}
        \mathrm{CE}(f) = \left( \mathbb{E}\left[ \left( \mathbb{E}\left[C\mid f(Y)\right] - f(Y)\right)^2\right] \right)^{1/2}
    \end{equation}
\end{restatable}

\begin{restatable}{corollary}{commonMode}
\label{cor:commonMode}
Suppose the two surrogate log-errors have equal means and variances $\sigma^2$, with correlation $\rho$, and the class prior is correctly specified, then
\begin{equation}
    \mathrm{CE}(\hat\eta) \leq \frac{\sigma}{2\sqrt{2}} \sqrt{1-\rho}
\end{equation}
\end{restatable}

Thus, shared surrogate error yields a tighter calibration bound when the two errors are strongly correlated, with the calibration error vanishing to zero as $\rho\rightarrow1$ (\Cref{cor:commonMode}).
Such a positive correlation is plausible because both $q_1$ and $q_0$ score the same statement using the same underlying LM, and may therefore inherit common statement-level effects such as lexical frequency, surface form, or other model-specific likelihood biases \citep{holtzman-etal-2021-surface, lovering-etal-2025-language}.
These effects are known to systematically influence LM probabilities, although whether they manifest as shared error across our two conditionals is ultimately an empirical question.
\Cref{prop:robustCal} provides the general bound allowing unequal means, variances, and prior misspecification.

\begin{table*}[ht]
  \centering
  \setlength{\tabcolsep}{2pt}
  \renewcommand{\arraystretch}{1.2}
  \begin{subtable}{\textwidth}
  \resizebox{\textwidth}{!}{%
  \begin{tabular}{@{}l*{9}{c}@{}}
    \toprule
    \textbf{Method} & \textbf{en} & \textbf{fr} & \textbf{de} & \textbf{hi} & \textbf{it} & \textbf{pt} & \textbf{es} & \textbf{th} & \textbf{Mean} \\
    \midrule
    \addlinespace[1pt]
    MSP & \bestmse{.1230}{.0080} & \bestmse{.1118}{.0073} & \bestmse{.1096}{.0068} & \mse{.1712}{.0052} & \bestmse{.1149}{.0070} & \bestmse{.1171}{.0071} & \bestmse{.1096}{.0067} & \mse{.1712}{.0055} & \secondmse{.1285}{.0067} \\
    Energy & \mse{.1726}{.0072} & \mse{.1463}{.0081} & \mse{.1450}{.0078} & \secondmse{.1454}{.0092} & \mse{.1424}{.0081} & \mse{.1536}{.0077} & \mse{.1432}{.0077} & \secondmse{.1573}{.0068} & \mse{.1507}{.0078} \\
    Entropy & \mse{.3065}{.0031} & \mse{.3110}{.0032} & \mse{.3030}{.0033} & \mse{.3238}{.0035} & \mse{.3031}{.0037} & \mse{.3301}{.0029} & \mse{.3349}{.0030} & \mse{.3068}{.0034} & \mse{.3149}{.0033} \\
    LN-Entropy & \secondmse{.1257}{.0096} & \mse{.1481}{.0079} & \mse{.1425}{.0082} & \mse{.1467}{.0094} & \mse{.1333}{.0092} & \mse{.1321}{.0095} & \secondmse{.1366}{.0082} & \mse{.1855}{.0064} & \mse{.1438}{.0086} \\
    NLL & \mse{.2937}{.0037} & \mse{.2992}{.0032} & \mse{.3016}{.0033} & \mse{.3088}{.0035} & \mse{.3072}{.0036} & \mse{.3029}{.0034} & \mse{.2936}{.0036} & \mse{.2813}{.0041} & \mse{.2986}{.0036} \\
    Perplexity & \mse{.4781}{.0012} & \mse{.4797}{.0012} & \mse{.4887}{.0010} & \mse{.4389}{.0019} & \mse{.4854}{.0010} & \mse{.4862}{.0011} & \mse{.4804}{.0011} & \mse{.4815}{.0010} & \mse{.4773}{.0012} \\
    \textsc{P(True)} & \mse{.3485}{.0037} & \mse{.3455}{.0036} & \mse{.3400}{.0030} & \mse{.3377}{.0028} & \mse{.3425}{.0034} & \mse{.3356}{.0034} & \mse{.3348}{.0034} & \mse{.3302}{.0028} & \mse{.3394}{.0033} \\
    \textsc{P(True)}-norm & \mse{.2359}{.0059} & \mse{.2205}{.0070} & \mse{.2203}{.0066} & \mse{.1630}{.0060} & \mse{.2118}{.0068} & \mse{.2117}{.0064} & \mse{.2252}{.0064} & \mse{.1625}{.0068} & \mse{.2063}{.0065} \\
    Gen-Z\textsuperscript\textdagger & \mse{.3456}{.0110} & \mse{.3863}{.0125} & \mse{.3809}{.0122} & \mse{.4613}{.0127} & \mse{.3933}{.0121} & \mse{.3847}{.0125} & \mse{.3646}{.0117} & \mse{.4527}{.0135} & \mse{.3962}{.0123} \\
    \rowcolor{oursgray} \textbf{CFP (ours)} & \mse{.1584}{.0098} & \secondmse{.1309}{.0109} & \secondmse{.1207}{.0105} & \bestmse{.0599}{.0094} & \secondmse{.1226}{.0110} & \secondmse{.1228}{.0108} & \mse{.1401}{.0103} & \bestmse{.0643}{.0099} & \bestmse{.1150}{.0103} \\
    \bottomrule
  \end{tabular}
  }
  \caption{Expected Calibration Error (ECE) ($\downarrow$)}
  \end{subtable}

  \begin{subtable}{\textwidth}
  \resizebox{\textwidth}{!}{%
  \begin{tabular}{@{}l*{9}{c}@{}}
    \addlinespace[5pt]
    \midrule
    MSP & \mse{.2625}{.0029} & \secondmse{.2558}{.0030} & \secondmse{.2574}{.0027} & \mse{.2804}{.0024} & \secondmse{.2578}{.0030} & \secondmse{.2587}{.0026} & \secondmse{.2566}{.0026} & \mse{.2832}{.0025} & \secondmse{.2641}{.0027} \\
    Energy & \mse{.2809}{.0034} & \mse{.2708}{.0041} & \mse{.2696}{.0038} & \mse{.2739}{.0037} & \mse{.2697}{.0040} & \mse{.2732}{.0040} & \mse{.2689}{.0037} & \mse{.2778}{.0030} & \mse{.2731}{.0037} \\
    Entropy & \mse{.3422}{.0027} & \mse{.3432}{.0027} & \mse{.3399}{.0027} & \mse{.3567}{.0027} & \mse{.3388}{.0028} & \mse{.3563}{.0026} & \mse{.3598}{.0026} & \mse{.3479}{.0027} & \mse{.3481}{.0027} \\
    LN-Entropy & \secondmse{.2620}{.0036} & \mse{.2678}{.0041} & \mse{.2669}{.0037} & \secondmse{.2712}{.0042} & \mse{.2638}{.0045} & \mse{.2629}{.0037} & \mse{.2652}{.0037} & \mse{.2894}{.0030} & \mse{.2687}{.0038} \\
    NLL & \mse{.3121}{.0024} & \mse{.3251}{.0028} & \mse{.3268}{.0027} & \mse{.3454}{.0027} & \mse{.3320}{.0028} & \mse{.3276}{.0030} & \mse{.3195}{.0029} & \mse{.3276}{.0031} & \mse{.3270}{.0028} \\
    Perplexity & \mse{.4739}{.0014} & \mse{.4763}{.0015} & \mse{.4870}{.0010} & \mse{.4422}{.0018} & \mse{.4837}{.0011} & \mse{.4845}{.0012} & \mse{.4772}{.0013} & \mse{.4811}{.0010} & \mse{.4757}{.0013} \\
    P(True) & \mse{.3799}{.0017} & \mse{.3829}{.0014} & \mse{.3786}{.0015} & \mse{.3869}{.0015} & \mse{.3775}{.0015} & \mse{.3719}{.0015} & \mse{.3727}{.0015} & \mse{.3816}{.0015} & \mse{.3790}{.0015} \\
    P(True)-norm & \mse{.2664}{.0007} & \mse{.2680}{.0007} & \mse{.2685}{.0007} & \mse{.2756}{.0005} & \mse{.2658}{.0007} & \mse{.2674}{.0007} & \mse{.2645}{.0006} & \secondmse{.2738}{.0006} & \mse{.2688}{.0007} \\
    Gen-Z\textsuperscript\textdagger & \mse{.3456}{.0110} & \mse{.3863}{.0125} & \mse{.3809}{.0122} & \mse{.4613}{.0127} & \mse{.3933}{.0121} & \mse{.3847}{.0125} & \mse{.3646}{.0117} & \mse{.4527}{.0135} & \mse{.3962}{.0123} \\
    \rowcolor{oursgray} \textbf{CFP (ours)} & \bestmse{.2306}{.0010} & \bestmse{.2400}{.0008} & \bestmse{.2395}{.0008} & \bestmse{.2481}{.0004} & \bestmse{.2407}{.0008} & \bestmse{.2411}{.0008} & \bestmse{.2379}{.0008} & \bestmse{.2462}{.0006} & \bestmse{.2405}{.0007} \\
    \bottomrule
  \end{tabular}
  }
  \caption{Brier's Score (BS) ($\downarrow$)}
  \end{subtable}

  \begin{subtable}{\textwidth}
  \resizebox{\textwidth}{!}{%
  \begin{tabular}{@{}l*{9}{c}@{}}
    \addlinespace[5pt]
    \midrule
    MSP & \mse{.5522}{.0164} & \mse{.5877}{.0177} & \mse{.5748}{.0152} & \mse{.5524}{.0180} & \mse{.5804}{.0170} & \mse{.5738}{.0160} & \mse{.5747}{.0150} & \mse{.5383}{.0164} & \mse{.5668}{.0165} \\
    Energy & \mse{.5545}{.0166} & \mse{.5593}{.0182} & \mse{.5663}{.0162} & \mse{.5319}{.0177} & \mse{.5560}{.0180} & \mse{.5546}{.0175} & \mse{.5656}{.0164} & \mse{.5303}{.0155} & \mse{.5523}{.0170} \\
    Entropy & \mse{.5568}{.0134} & \mse{.5783}{.0154} & \mse{.5663}{.0141} & \mse{.5367}{.0130} & \mse{.5660}{.0144} & \mse{.5650}{.0134} & \mse{.5629}{.0142} & \mse{.5175}{.0155} & \mse{.5562}{.0142} \\
    LN-Entropy & \mse{.5565}{.0163} & \mse{.5795}{.0184} & \mse{.5736}{.0159} & \mse{.5504}{.0181} & \mse{.5724}{.0176} & \mse{.5691}{.0164} & \mse{.5720}{.0156} & \mse{.5368}{.0160} & \mse{.5638}{.0168} \\
    NLL & \mse{.7250}{.0126} & \mse{.6680}{.0150} & \mse{.6608}{.0126} & \mse{.5504}{.0137} & \mse{.6506}{.0144} & \mse{.6642}{.0147} & \mse{.6746}{.0142} & \mse{.5598}{.0162} & \mse{.6442}{.0142} \\
    Perplexity & \mse{.7235}{.0148} & \mse{.6566}{.0176} & \mse{.6601}{.0152} & \mse{.5598}{.0169} & \mse{.6471}{.0185} & \mse{.6494}{.0162} & \mse{.6725}{.0158} & \mse{.5843}{.0172} & \mse{.6442}{.0165} \\
    P(True) & \mse{.6184}{.0149} & \mse{.5896}{.0153} & \mse{.6007}{.0156} & \mse{.5426}{.0160} & \mse{.6062}{.0150} & \mse{.6107}{.0149} & \mse{.6016}{.0154} & \mse{.5577}{.0161} & \mse{.5909}{.0154} \\
    P(True)-norm & \secondmse{.7444}{.0117} & \secondmse{.7077}{.0134} & \secondmse{.7071}{.0127} & \bestmse{.6404}{.0156} & \bestmse{.7091}{.0130} & \bestmse{.6970}{.0126} & \secondmse{.7197}{.0124} & \bestmse{.6284}{.0158} & \bestmse{.6942}{.0134} \\
    Gen-Z\textsuperscript\textdagger & \mse{.6557}{.0109} & \mse{.6145}{.0125} & \mse{.6196}{.0123} & \mse{.5394}{.0127} & \mse{.6076}{.0121} & \mse{.6161}{.0125} & \mse{.6358}{.0117} & \mse{.5484}{.0135} & \mse{.6046}{.0123} \\
    \rowcolor{oursgray} \textbf{CFP (ours)} & \bestmse{.7780}{.0117} & \bestmse{.7135}{.0142} & \bestmse{.7081}{.0135} & \secondmse{.5806}{.0161} & \secondmse{.7009}{.0144} & \secondmse{.6968}{.0144} & \bestmse{.7309}{.0133} & \secondmse{.6013}{.0176} & \secondmse{.6888}{.0144} \\
    \bottomrule
  \end{tabular}%
  }
  \caption{Area Under the ROC curve (AUROC) ($\uparrow$)}
  \end{subtable}

  \caption{Main results by language, with rightmost column representing mean across all languages.
  Each cell reports mean with five-fold cross-validation standard error (SE).
  \rev{\textsuperscript{\textdagger} Gen-Z outputs hard predictions, so ECE and BS reduce to classification error and are reported only for completeness.}
  }
  \label{tab:main-results-language}
\end{table*}

\section{Evaluating Calibration and Discrimination}
\label{sec:expt}

\subsection{Experimental Setup}

\paragraph{Datasets.} We utilise the \textbf{Multilingual True-False} \citep{xiao-etal-2025-analyzing} datasets, which comprise factually true and false statements from six different topics and eight languages, including English (en), French (fr), German (de), Hindi (hi), Italian (it), Portuguese (pt), Spanish (es), and Thai (th). In total, there are $48,680$ statements, with an average of $6,085$ per language. Each statement has a matching translation in each language. This crucially allows us to evaluate performance fairly on semantically equivalent datasets and to study our cross-lingual factual consistency alignment proposal in \Cref{sec:fact_align}.

\paragraph{Baselines.} We compare our proposed CFP score against nine baselines, spanning three families of methods, namely
\begin{inparaenum}[(i)]
    \item marginal uncertainty: Maximum Sequence Probability (MSP), Energy, Entropy, Length-Normalised (LN) Entropy, Negative Log-Likelihood (NLL), and Perplexity \citep{huang2025look};
    \item direct elicitation: P(True) \citep{kadavath2022languagemodelsmostlyknow} and P(True)-norm \citep{heo2025uncertainty}; and
    \item generative classification, Gen-Z \citep{kumar2024gen}.
\end{inparaenum}
For non-probabilistic scores which are not in the $[0, 1]$ range, we normalise using $\operatorname{min-max}$ linear scaling (see~\Cref{sec:supp_expt}).
We adapt Gen-Z, originally proposed for zero-shot text classification, to binary factuality by treating factual and non-factual conditions as candidate classes and predicting the class whose condition assigns the claim greater likelihood. This provides a particularly similar baseline: it uses the same reverse-generative principle as CFP, but is formulated as a categorical classifier rather than as a continuous measure of confidence.
\rev{We report ECE and Brier's Score for completeness, but note that under Gen-Z hard predictions, these metrics collapse to classification error and should not be interpreted as directly comparable measures of calibration.}

\paragraph{Metrics.} We evaluate calibration with Expected Calibration Error \citep[ECE;][]{pmlr-v70-guo17a}, using 10 equal-width bins, discrimination with Area Under the Receiver-Operating Characteristic curve (AUROC), and both with Brier's Score \citep[BS;][]{glenn1950verification}.
\rev{Because ECE can be low for weakly discriminative predictions, we interpret calibration jointly with Brier's Score and AUROC rather than treating ECE in isolation.}

\paragraph{Language Models.} We test on three families of models:
\begin{inparaenum}[(i)]
\item Qwen-3 (0.6B, 1.7B, 4B, 8B, 14B) \citep{yang2025qwen3};
\item Qwen-3.5 (0.8B, 2B, 4B, 9B) \citep{qwen3.5_2026}; and
\item Gemma-4 (2B, 4B, 12B) \citep{team2026gemma}.
\end{inparaenum}
Furthermore, for each model, we test both the base and instruction-tuned versions, resulting in a total of 12 pairs of LMs, ranging from 0.6 to 14B in size.

\subsection{Results and Discussion}

\begin{table}[htbp]
    \centering
    \setlength{\tabcolsep}{6pt}
    \renewcommand{\arraystretch}{1.25}
    \resizebox{\linewidth}{!}{%
    \begin{tabular}{lccc}
    \toprule
    \textbf{Component} & \textbf{ECE ($\downarrow$)} & \textbf{BS ($\downarrow$)} & \textbf{AUROC ($\uparrow$)} \\
    \midrule
    \textbf{$\log q_1$} & \mse{.2903}{.0035} & \mse{.3235}{.0027} & \mse{.6309}{.0142}\\
    \textbf{$\log q_0$} & \mse{.2927}{.0035} & \mse{.3286}{.0028} & \mse{.6032}{.0148}\\
    \textbf{CFE$_{\text{raw}}$} & \mse{.1420}{.0084} & \mse{.2369}{.0025} & \mse{.6894}{.0144} \\
    \textbf{CFE$_\text{single}$} & \mse{.1373}{.0079} & \mse{.2411}{.0023} & \mse{.6729}{.0146} \\
    \textbf{CFE} & \mse{.1446}{.0080} & \mse{.2387}{.0024} & \mse{.6888}{.0144} \\
    \textbf{CFP} & \mse{.1150}{.0103} & \mse{.2405}{.0007} & \mse{.6888}{.0144} \\
    \bottomrule
    \end{tabular}
    }
    \caption{Ablation results with different formulation choices.}
    \label{tab:ablation}
\end{table}

\paragraph{Main Results.}
\Cref{tab:main-results-language} presents overall performance across topics, models, and prompts.
On average, \rev{CFP achieves the lowest mean ECE (0.115) and Brier's Score (0.241),} while maintaining competitive discrimination with an AUROC of 0.689, compared with 0.694 for P(True)-norm.
\rev{However, the ECE advantage is not uniform across languages: MSP attains lower ECE in six of eight languages, while CFP's lower aggregate ECE is driven primarily by Hindi and Thai. Moreover, in these two languages, CFP's low ECE coincides with weak discrimination and Brier's Score close to 0.25, the value attained by a constant predictor. This suggests that low ECE partly reflects confidence compression rather than a stronger factuality signal.}
\rev{Importantly, weak discrimination in Hindi and Thai occurs broadly across methods, suggesting a limitation of the underlying model signal in these languages rather than a failure specific to CFP.}
This cross-lingual gap motivates \Cref{sec:fact_align}, where we investigate whether stronger English factual evidence can be transferred to target languages.

\paragraph{Ablation.}
\Cref{tab:ablation} isolates three design choices underlying CFP. First, we evaluate the two conditional likelihoods, $\log q_1$ and $\log q_0$, separately to test whether their contrast is beneficial. We further compare raw likelihood (CFE$_\text{raw}$) against token-length normalisation (CFE), and a single pair of truth conditions (CFE$_\text{single}$) against averaging over five paraphrased conditions (CFE). Non-probabilistic scores are linearly scaled to $[0,1]$ for calibration evaluation.
Contrasting the two conditionals provides a significant improvement, with CFE reducing ECE from approximately $0.29$ to $0.14$ and increasing AUROC from $0.60/0.63$ to $0.69$.
Length normalisation has only a modest effect, while averaging across paraphrased conditions slightly improves discrimination over the single-prompt variant. 
Finally, mapping CFE to CFP with the sigmoid substantially reduces ECE \rev{while leaving AUROC unchanged}.
Overall, the main gain arises from contrasting the competing likelihoods while other auxiliary choices have limited effects.

\paragraph{Post-Training and Prompt Template.}
In \Cref{tab:template_posttrain}, we provide a breakdown of instruction-tuned (IT) versus base models, and with versus without using chat template.
We note  that the unnatural combination of IT and Chat is only illustrated for completeness, and other results do not include this configuration.
\rev{Post-training and chat templating consistently reduces discrimination, but their effects on calibration are more mixed.
For CFP, ECE decreases from 0.140 to 0.088 between base/no-chat and IT/chat settings even as AUROC falls from 0.736 to 0.652.
CFP nevertheless exhibits a smaller drop in discrimination than P(True)-norm across these configurations.}

\begin{figure}
    \centering
    \includegraphics[width=\linewidth]{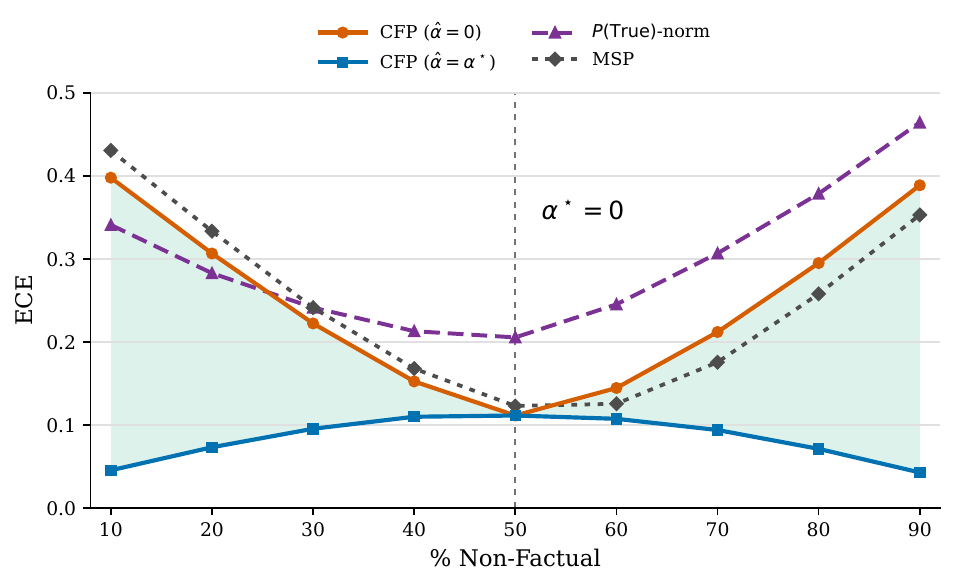}
    \caption{ECE for varying non-factual proportions with fixed $\hat{\alpha}=0$ and corrected $\hat{\alpha}=\alpha^\star$.}
    \label{fig:prior_misspec}
\end{figure}

\paragraph{Prior Misspecification.} To assess the effects on CFP's calibration due to prior misspecification, we simulate settings in which the proportion of non-factual claims varies from 10\% to 90\%.
For each prevalence, we construct five subsets without replacement and compare the uncorrected posterior, which assumes equal priors ($\hat{\alpha}=0$), against the corrected posterior using the ideal prior log-odds.
As \Cref{fig:prior_misspec} shows, CFP's calibration is sensitive to the prior, with ECE increasing from $0.1$ under the original balanced class setting to $0.4$ under the extreme case.
Even though CFP frequently has better calibration than P(True)-norm, this sensitivity represents a limitation when the class prior is unknown.
On the other hand, when prior knowledge is available, ECE reduces to as low as 0.04, without any retraining and leaving AUROC unchanged.

\section{Geometry-Guided Cross-Lingual Factuality Alignment}
\label{sec:fact_align}

CFE discriminates factual from non-factual statements reliably in English, but its performance degrades substantially in lower-resource languages.
This gap motivates the question: \emph{can stronger English factual evidence be transferred to a target language without requiring target-language factuality labels?}
To achieve this, for a translated claim with the same underlying proposition, we would ideally preserve the factual evidence assigned to its English counterpart, while allowing language-specific differences in absolute likelihood.

\begin{figure}[htbp]
    \centering
    \includegraphics[width=\linewidth]{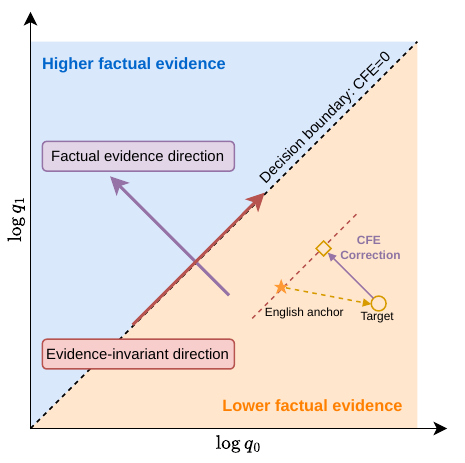}
    \caption{Geometry of CFE alignment.}
    \label{fig:cfe_align}
\end{figure}

\subsection{Likelihood Geometry of Factual Evidence}

For a claim $y$, the two surrogate conditional likelihoods induce a two-dimensional representation $\bm{z}(y) = [\log q_0(y), \,\log q_1(y)]^\intercal$. CFE can be written as the projection onto the unit direction $\bm{t}$
\begin{equation}
    \hat E(y) = \sqrt{2} \bm{t}^\intercal \bm{z}, \;\bm{t}=\frac{1}{\sqrt 2}
    \begin{pmatrix}
        -1\\1
    \end{pmatrix}
\end{equation}

As illustrated in \Cref{fig:cfe_align}, movement along $\bm{t}$ changes factual evidence, while the orthogonal $\bm{m}=1/\sqrt2 (1,\, 1)^\intercal$ is evidence-invariant. The decision boundary $\hat{E}(y)=0$ is therefore the diagonal $\log q_1 = \log q_0$.

Related work has treated LM log-likelihood as geometric coordinates \citep{oyama-etal-2025-mapping} and has used controlled probability-space comparisons to isolate linguistic attributes \citep{hu-etal-2026-string}.
In our setting, both coordinates instead score the same claim under competing factual and non-factual conditions, analytically inducing a factual-evidence and an orthogonal evidence-invariant direction.

\subsection{Factuality Alignment}

In the induced likelihood geometry, we treat the stronger English CFE as an anchor and align the target-language claim only along the factual-evidence direction $\bm{t}$. This selectively transfers factuality evidence while leaving variation along the evidence-invariant direction $\bm{m}$ unconstrained.

\paragraph{Objective.}
Given parallel English ($\text{en}$) and target-language $l$ claims $\{(y_i^\text{en}, y_i^l)\}_{i=1}^N$, we use the frozen base model's English CFE as an anchor and optimise LoRA parameters $\phi$
\begin{equation}
    \frac{1}{N} \sum_{i=1}^{N}
    \left[\hat{E}_{\phi}(y_i^{l}) - \operatorname{sg}\left(\hat{E}(y_i^\text{en})\right)\right]^2,
\end{equation}
where $\operatorname{sg}$ is the stop-gradient operator \citep{chen2021exploring} to treat the English CFE as a constant and to ensure the adaptation transfers the stronger English factuality signal to the target language.

\begin{table}[htbp]
    \centering
    \setlength{\tabcolsep}{7pt}
    \renewcommand{\arraystretch}{1.2}

    \begin{subtable}{\linewidth}
        \centering
        \resizebox{\linewidth}{!}{%
        \begin{tabular}{llccc}
            \toprule
            \multirow{2}{*}{\textbf{Language}} & \multirow{2}{*}{\textbf{Objective}}& \multirow{2}{*}{\textbf{True--False}} & \multicolumn{2}{c}{\textbf{Belebele}} \\
            \cmidrule(lr){4-5}
            & & & \textbf{Target} & \textbf{English} \\
            \midrule

            English & Base & 83.9 & --- & 86.9 \\
            \midrule
            \multirow{3}{*}{German} & Base & 77.4 & 81.9 & 86.9 \\
            & Full & 69.9 \footnotesize \textcolor{darkred}{$(-7.5)$} & 81.7 \footnotesize \textcolor{darkred}{$(-0.2)$} & 86.9 \footnotesize $(0.0)$ \\
            & CFE & 82.2 \footnotesize \textcolor{darkgreen}{$(+4.7)$} & 82.0 \footnotesize \textcolor{darkgreen}{$(+0.1)$} & 87.0 \footnotesize \textcolor{darkgreen}{$(+0.1)$} \\
            \midrule
            
            \multirow{3}{*}{Hindi} & Base & 51.4 & 42.8 & 86.9 \\
            & Full & 50.6 \footnotesize \textcolor{darkred}{$(-0.8)$} & 40.2 \footnotesize \textcolor{darkred}{$(-2.6)$} & 86.9 {\footnotesize $(0.0)$} \\
            & CFE & 57.1 \footnotesize \textcolor{darkgreen}{$(+5.7)$} & 41.7 \footnotesize \textcolor{darkred}{$(-1.1)$} & 86.8 \footnotesize \textcolor{darkred}{$(-0.1)$} \\

            \bottomrule
        \end{tabular}%
        }
        \caption{Qwen-3.5-2B}
    \end{subtable}

    \vspace{6pt}

    \begin{subtable}{\linewidth}
        \centering
        \resizebox{\linewidth}{!}{%
        \begin{tabular}{llccc}
            \toprule
            \multirow{2}{*}{\textbf{Language}} & \multirow{2}{*}{\textbf{Objective}}& \multirow{2}{*}{\textbf{True--False}} & \multicolumn{2}{c}{\textbf{Belebele}} \\
            \cmidrule(lr){4-5}
            & & & \textbf{Target} & \textbf{English} \\
            \midrule
            
            English & Base & 87.9 & --- & 92.9 \\
            \midrule
            \multirow{3}{*}{German}
            & Base & 85.0 & 92.4 & 92.9 \\
            & Full & 75.9 \textcolor{darkred}{\footnotesize $(-9.0)$} & 92.1 \footnotesize \textcolor{darkred}{$(-0.3)$} & 93.0 \footnotesize \textcolor{darkgreen}{$(+0.1)$} \\
            & CFE & 87.4 \footnotesize \textcolor{darkgreen}{$(+2.4)$} & 92.4 {\footnotesize $(0.0)$} & 93.1 \footnotesize \textcolor{darkgreen}{$(+0.2)$} \\
            \midrule

            \multirow{3}{*}{Hindi} & Base & 57.0 & 65.3 & 92.9 \\
            & Full & 57.1 \footnotesize \textcolor{darkgreen}{$(+0.1)$} & 64.7 \footnotesize \textcolor{darkred}{$(-0.6)$} & 93.0 \footnotesize \textcolor{darkgreen}{$(+0.1)$} \\
            & CFE & 70.5 \footnotesize \textcolor{darkgreen}{$(+13.5)$} & 65.9 \footnotesize \textcolor{darkgreen}{$(+0.6)$} & 93.0 \footnotesize \textcolor{darkgreen}{$(+0.1)$} \\
            \bottomrule
        \end{tabular}%
        }
        \caption{Qwen-3.5-4B}
    \end{subtable}

    \caption{Alignment results on True-False discriminability (AUROC) and general-task retention on Belebele (Accuracy). 
    ``Full'' and ``CFE'' are finetuned on True-False training sets with full-consistency and CFE objective, respectively.
    \rev{English Base row reports the frozen English reference used as alignment anchor.}
    }
    \label{tab:ft_discriminative}
\end{table}

\paragraph{Baseline.}

Motivated by consistency regularisation methods that encourage semantically equivalent or translated inputs to induce similar model predictions \citep{elazar-etal-2021-measuring, zheng-etal-2021-consistency, gao-etal-2023-improving}, we construct a full-consistency baseline that aligns both likelihood coordinates

\begin{equation}
\frac{1}{2N} \sum_{i=1}^{N} \sum_{c \in\{0,1\}}\left[ \log q_c(y_i^l) - \operatorname{sg} (\log q_c(y_i^\text{en}))\right]^2
\end{equation}

This objective is analogous to existing full-distribution consistency objectives, but operates on the two conditional log-likelihoods. It therefore constrains both the factual-evidence and evidence-invariant directions, allowing us to test whether CFE's selective alignment in the factual-evidence direction is beneficial.

\paragraph{Finetuning Setting.} We consider two base models, Qwen-3.5-2B and Qwen-3.5-4B. For each target language (German for high-resource and Hindi for low-resource evaluation) and True-False topic, we train a separate LoRA adapter \citep{hu2022lora}. We adopt rank $r=8$, scaling factor $\alpha=16$, and dropout 0.05, and train for at most five epochs with learning rate $3\times10^{-5}$, weight decay 0.01, and batch size 16.

\paragraph{Remark.} Our goal is not to compare general cross-lingual consistency objectives, but to test a specific implication of the CFE geometry: because factual evidence depends only on the likelihood-ratio coordinate, cross-lingual transfer need only constrain this direction rather than the full conditional-likelihood representation.

\subsection{Experimental Results}

We evaluate alignment along three dimensions. First, to test whether it recovers degraded target-language factual discrimination, we report AUROC on held-out True-False statements, averaged across topics (\Cref{tab:ft_discriminative}). At the same time, we evaluate whether factuality alignment preserves broader language capability with accuracy on Belebele \citep{bandarkar-etal-2024-belebele} in both English and the corresponding target language.
Finally, we assess open-ended factual generation through biography generation.
We sample 50 entities from FActScore \citep{min-etal-2023-factscore}, prompt the model to generate a biography for each using up to a maximum of $512$ tokens, and measure factual precision (\Cref{tab:ft_generative}).
We provide more details in \Cref{subsec:fact_setup}.

\begin{table}[htbp]
    \centering
    \setlength{\tabcolsep}{3pt}
    \renewcommand{\arraystretch}{1.25}

    \resizebox{\linewidth}{!}{%
    \begin{tabular}{llccc}
        \toprule
        \textbf{Language} & \textbf{Objective}& \textbf{Factual Precision (\%)} & \textbf{\# Claims} & \textbf{\# Correct}\\
        \midrule
        \multirow{3}{*}{English} & Base & 41.29 & 12.74 & 5.26 \\
        & Full & 38.93 \footnotesize \textcolor{darkred}{(-2.36)} & 12.87 \footnotesize \textcolor{darkgreen}{(+0.13)} & 5.01 \footnotesize \textcolor{darkred}{(-0.25)} \\
        & CFE & 42.32 \footnotesize \textcolor{darkgreen}{(+1.03)} & 11.67 \footnotesize \textcolor{darkred}{(-1.07)} & 4.94 \footnotesize \textcolor{darkred}{(-0.32)} \\

        \midrule
        \multirow{3}{*}{Hindi} & Base & 27.47 & 11.72 & 3.22 \\
        & Full & 32.42 \footnotesize \textcolor{darkgreen}{(+4.95)} & 8.29 \footnotesize \textcolor{darkred}{(-3.43)} & 2.69 \footnotesize \textcolor{darkred}{(-0.53)} \\
        & CFE & 32.88 \footnotesize \textcolor{darkgreen}{(+5.41)} & 10.92 \footnotesize \textcolor{darkred}{(-0.80)} & 3.59 \footnotesize \textcolor{darkgreen}{(+0.37)} \\
        \bottomrule
    \end{tabular}%
    }
    \caption{Long-form biography generation results using Qwen-3.5-4B, base and finetuned on Hindi statements.}
    \label{tab:ft_generative}
\end{table}

As shown in \Cref{tab:ft_discriminative} (breakdown in \Cref{tab:align_breakdown2b,tab:align_breakdown4b}), evidence-selective CFE alignment consistently improves target-language factual discrimination with minimal effects on general-task performance.
\rev{Relative to the frozen English reference, CFE alignment improves considerably.
For example, with Qwen-3.5-4B finetuned on Hindi using CFE, AUROC increases from 57.0 to 70.5, closing 44\% of the gap to English's 87.9.}
In comparison, full consistency does not reliably narrow this gap and substantially degrades German discrimination.
These results support our hypothesis that constraining only the factual-evidence direction can be preferable to aligning the complete likelihood representation.
We observe a similar trend in the biography generation task (\Cref{tab:ft_generative}; breakdown in \Cref{tab:align_bio_breakdown}).
\rev{Both CFE and full consistency improve factual precision by a similar amount, but their generation behaviour differs.
Full consistency produces substantially fewer claims and fewer correct claims per entity, whereas CFE alignment approximately preserves the number of claims generated while increasing the number of correct claims.}
Thus, these results suggest that evidence-selective alignment can improve precision while better preserving factual output volume.

\section{Conclusion}

In this work, we study contrastive generative likelihoods as factual evidence, instantiating their likelihood difference as CFE and the probabilistic CFP.
\rev{Empirically, CFP is competitive with existing likelihood baselines across both calibration and discrimination, although its relative advantage depends on the language and evaluation setting.
In lower-resource languages, low ECE coincides with weak discrimination.} %
Motivated by this, we interpret the paired conditional likelihoods in CFE geometrically, and derive a selective evidence alignment objective that transfers stronger English evidence to target languages while leaving evidence-irrelevant variation unconstrained. %
In our experiments, this selective alignment improves target-language discrimination and factual generation while largely preserving general-task performance.
Taken together, our results highlight the importance of evaluating factual confidence through both discrimination and calibration perspectives, and suggest that isolating the evidence-bearing component of likelihoods provides a useful route toward cross-lingual factuality transfer.

\section*{Limitations}

Our  work has several limitations. First, the multilingual True-False benchmark consists primarily of short, atomic factual statements.
This controlled setting is well suited to evaluating semantically matched claims across languages, but does not capture more complex or long-form factuality; such extension remains future work.
Second, our main evaluation covers LLMs only up to 14B parameters, so generalisation to much larger models is plausible but untested.
Finally, our cross-lingual alignment experiments are intentionally a proof of concept.
Due to computational constraints, we finetune only 2B-4B models and evaluate general-task retention with Belebele and open-ended factuality through biography generation.
These benchmarks provide useful but limited coverage of broader capabilities and generation settings. 
Biography factuality is additionally assessed using LLM-as-a-judge, which correlates well but imperfectly with human evaluation.

\bibliography{custom}

\appendix

\section{Theoretical Proofs and Analysis}
\label{sec:proofs}

\begin{restatable}{proposition}{robustCal}
\label{prop:robustCal}
Suppose $\epsilon_1(Y),\; \epsilon_0(Y)$ have finite second moments, with mean $\mu_c$, variance $\sigma_c^2$, and covariance $\kappa$. Then,
\begin{equation}
    \mathrm{CE}(\hat{\eta}) \leq \frac{1}{4} \sqrt{\left(\mu_1 - \mu_0 + \delta_\pi\right)^2+\sigma_1^2 + \sigma_0^2 -2\kappa}
\end{equation}
\end{restatable}

\begin{proof}

Let $\eta=P(C=1 \mid Y)$ denote the true factual posterior and $\hat{\eta}(y) = \sigma \left( \hat{E}(y) + \hat \alpha \right)$ the contrastive factual posterior. By definition of the surrogate log-errors,

\begin{equation*}
    \epsilon_c(y) = \log q_c(y) - \log p_c^*(y),
\end{equation*}
and therefore $\hat{E} - E^* = \epsilon_1 - \epsilon_0$.

We first relate approximation of the Bayes posterior to calibration. Using the population $L_2$ calibration error, we have
\begin{align*}
    \mathrm{CE}^2(\hat\eta) =&\mathbb{E}\left[ \left( \mathbb{E}\left[C\mid \hat\eta\right] - \hat\eta\right)^2\right] \\
    =& \mathbb E \left[ \left( \mathbb E[\eta\mid \hat\eta]-\hat\eta \right)^2 \right] \\
    =& \mathbb{E}\left[\mathbb{E} \left(\eta - \hat\eta \mid \hat\eta\right)^2\right]
\end{align*}
since, by the tower property of conditional expectation, $\mathbb{E}[C\mid \hat\eta] = \mathbb E \left[\mathbb{E} [C\mid y] \mid \hat\eta\right] = \mathbb E[\eta\mid \hat\eta]$.
Applying conditional Jensen's inequality to the convex $x\mapsto x^2$,

\begin{equation*}
    \mathbb{E} \left(\eta - \hat\eta \mid \hat\eta\right)^2 \leq \mathbb{E} \left[(\eta-\hat\eta)^2\mid \hat\eta\right]
\end{equation*}
Taking expectations and using the tower property again gives

\begin{equation}
\nonumber
\mathrm{CE}^2(\hat\eta) \leq \mathbb{E} \left[ (\eta - \hat\eta)^2\right]
\end{equation}

Next, we relate the surrogate log-errors to calibration. Note that the logistic sigmoid satisfies
\begin{equation*}
    \sigma'(x) = \sigma(x)(1-\sigma(x))\leq 1/4
\end{equation*}
By mean-value theorem, $\sigma$ is $1/4$-Lipschitz. Thus,
\begin{align*}
    |\hat\eta - \eta| &= \left|\sigma \left(\hat E + \hat\alpha \right) - \sigma \left(E^* + \alpha\right) \right| \\
    &\leq \frac{1}{4} \left| \left(\hat{E}-E^*\right) + \left(\hat \alpha - \alpha\right)\right| \\
    &= \frac{1}{4} \left| \left(\epsilon_1 - \epsilon_0\right) + \delta_\pi\right|
\end{align*}
Squaring and taking expectations,
\begin{equation*}
\mathrm{CE}^2 (\hat\eta) \leq \frac{1}{16} \mathbb{E}\left[\left(\epsilon_1 - \epsilon_0 + \delta_\pi\right)^2\right]
\end{equation*}
It remains to evaluate this second moment. Let $X=\epsilon_1-\epsilon_0+\delta_\pi$. Since $\delta_\pi$ is constant,
\begin{equation*}
    \mathbb E (X) = \mu_1 - \mu_0 + \delta_\pi
\end{equation*}
Moreover,
\begin{align*}
\mathbb{V}\text{ar}(X) &= \mathbb{V}\text{ar}(\epsilon_1-\epsilon_0)\\
&= \mathbb{V}\text{ar}(\epsilon_1) + \mathbb{V}\text{ar}(\epsilon_0) - 2\operatorname{Cov}(\epsilon_1,\epsilon_0)\\ &= \sigma_1^2+\sigma_0^2-2\kappa 
\end{align*}
With the identity $\mathbb E(X^2) = \mathbb{V}\text{ar}(X) + \mathbb E(X)^2$, we obtain
\begin{equation}
\mathrm{CE}^2(\hat\eta) \leq \frac{1}{16} \left[(\mu_1-\mu_0+\delta_\pi)^2 + \sigma_1^2 + \sigma_0^2 - 2\kappa\right]
\end{equation}
Finally, taking the non-negative square root gives us the desired inequality.
\end{proof}

\commonMode*
\begin{proof}
Under the matched means and variances assumption, $\mu_1=\mu_0$ and $\sigma_1^2=\sigma_0^2=\sigma^2$, and when the prior is correctly specified, $\delta_\pi=0$.

We can write correlation as $\rho=\operatorname{Cov}(\epsilon_1, \epsilon_0)/\sigma^2$.
Thus, \Cref{prop:robustCal} reduces to

\begin{align*}
    \mathrm{CE}(\hat\eta) &\leq \frac{1}{4}\sqrt{2\sigma^2-2\sigma^2\rho}\\
    &=\frac{\sigma}{2\sqrt{2}}\sqrt{1-\rho}
\end{align*}
\end{proof}

\section{Experimental Details}
\label{sec:expt_details}

\subsection{Language Model Formulation of CFE}
\label{subsec:cfe_formulation}

For each factuality class $c\in\{0,1\}$, we express the condition $z_c$ as a natural-language prompt, where $z_1$ describes a factual generative condition and $z_0$ non-factual (\Cref{fig:prompt_conditions}).
We provide $z_c$ as the model prefix and teacher-force the claim $y$ as its continuation. Here, $q_\theta(\cdot \mid \cdot)$ denotes the LM's next-token predictive distribution.
For a claim $y_{1:L}$ and the $k$-th prompt pair (\Cref{fig:prompt_conditions}), the surrogate conditional log-likelihood is
\begin{align}
    \ell_c^{(k)}(y) &= \log q_\theta(y\mid z_c^{(k)}) \\
    &= \sum_{t=1}^{L} \log q_\theta\left( y_t \mid z_c^{(k)}, y_{<t} \right)
\end{align}
We include the model-specific assistant end-of-turn token (or EOS token for raw-text models) as the final target token. Because summed log-likelihood differences can scale with claim length, we consider the length-normalised (LN) variant. Thus, the raw and LN evidence scores are
\begin{equation}
    \hat{E}_{\mathrm{raw}}^{(k)}(y) = \ell_1^{(k)}(y) - \ell_0^{(k)}(y), \,
    \hat{E}_{\mathrm{LN}}^{(k)}(y) =  \frac{\hat{E}_{\mathrm{raw}}^{(k)}(y)}{L}
\end{equation}
To reduce sensitivity to the particular wording of $z_c$, our default CFE additionally averages over the $K=5$ paraphrased condition pairs listed in \Cref{fig:prompt_conditions}.
\begin{equation}
\label{eq:cfe}
    \hat{E}(y) = \frac{1}{K}\sum_{k=1}^K \hat{E}_{\mathrm{LN}}^{(k)}(y)
\end{equation}
We refer to this averaging as prompt normalisation. In the ablation in \Cref{tab:ablation}, CFE$_\mathrm{raw}$ omits length normalisation, CFE$_{\mathrm{single}}$ uses only the first ($k=1$) condition pair, and the default CFE applies both operations (\Cref{eq:cfe}).

\subsection{Prompt Templates.}

We provide the prompt templates for CFE in \Cref{fig:prompt_conditions} and P(True)/P(True)-norm in \Cref{fig:prompt_ptrue}.
We collectively refer to the prompts in the respective target languages as native prompting.

\begin{figure*}
    \centering
    \includegraphics[width=\linewidth]{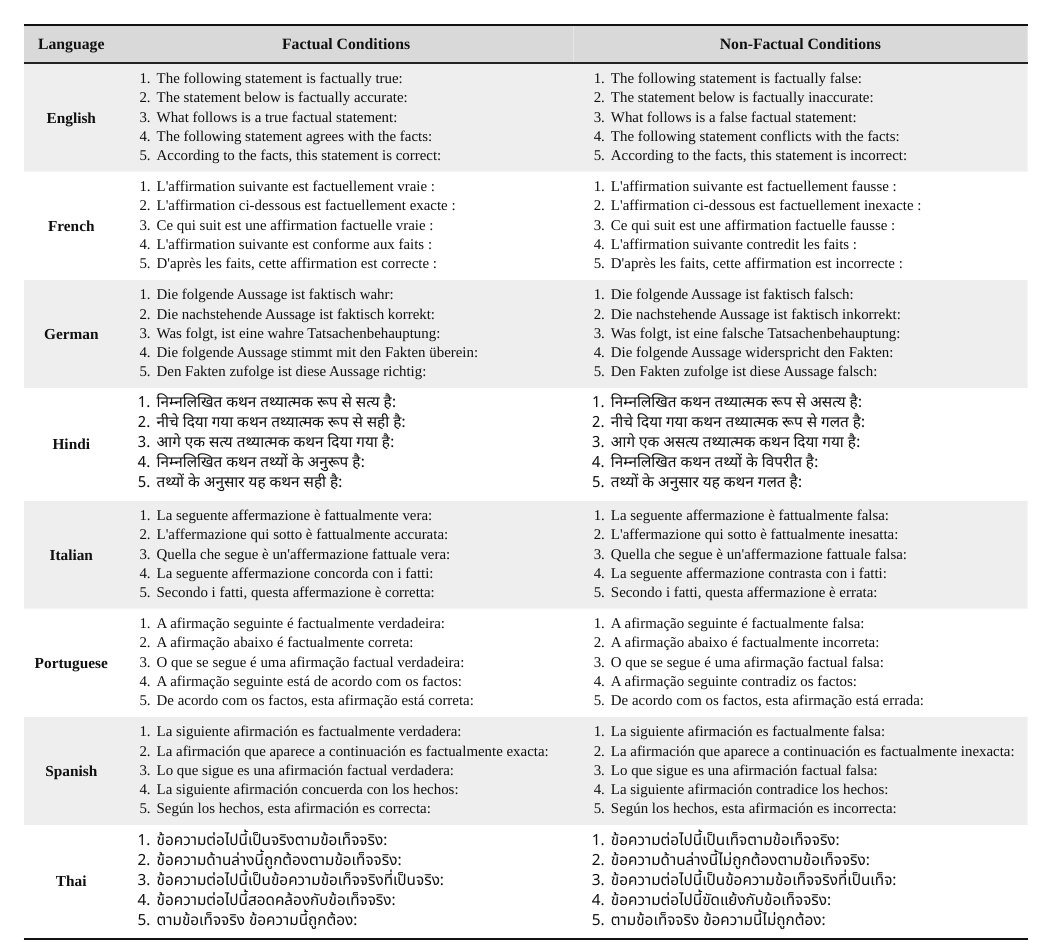}
    \caption{Factual and non-factual generative conditions for CFE. Pair \#1 is the default for CFE$_\text{single}$.}
    \label{fig:prompt_conditions}
\end{figure*}

\begin{figure}
    \centering
    \includegraphics[width=\linewidth]{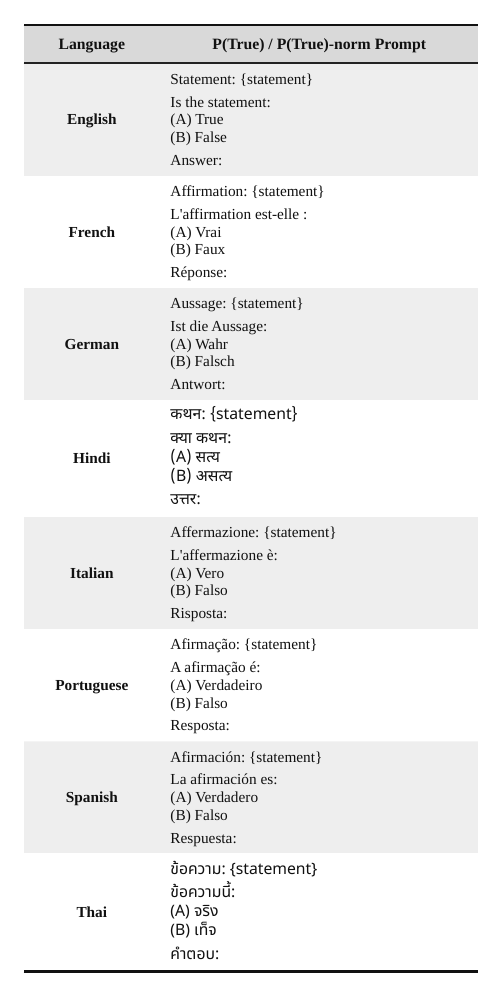}
    \caption{P(True) / P(True)-norm prompt.}
    \label{fig:prompt_ptrue}
\end{figure}

\subsection{Factuality Alignment Setup}
\label{subsec:fact_setup}

\paragraph{Finetuning Data.}
Each target--language True--False dataset contains 6,085 claims aligned one-to-one with their English translations, spanning six topics.
We train an independent adapter for each target-language--topic pair. For a topic containing $N$ paired claims, we split 70/30 for training and testing; we reserve 10\% of the training set for validation. Across six topic splits, this corresponds to 3,830, 428, and 1,827 pairs for training, validation, and testing, with each adapter using between 386 and 918 training pairs.
Although each topic provides only several hundred finetuning pairs, the experiment updates a small LoRA parameter subset. We therefore regard this setting as a controlled analysis for the proposed alignment geometry.

\paragraph{Belebele Evaluation.}
We evaluate general-task retention on the complete Belebele test split, comprising 900 multiple-choice reading-comprehension questions per language. For each adapter, we evaluate both English (\texttt{eng\_Latn}) and its target language (\texttt{deu\_Latn} or \texttt{hin\_Deva}) subsets before and after finetuning.
Following a five-shot setup, preceding every Belebele query, we provide the same five English demonstrations, selected deterministically with seed 42 from RACE \citep{lai-etal-2017-race}.
We use the raw-text ``\texttt{P:/Q:/A--D:/Answer:}'' format for passage, question, and the answer options, without a chat template.
We report accuracy over all 900 questions.

\paragraph{Biography Generation and Evaluation.}
We sample 50 unique people from FActScore \citep{min-etal-2023-factscore}.
We prompt each model in the target language to produce a single-paragraph factual biography with raw-text prompting (``Write a factual biography of \{\texttt{entity}\} in one paragraph.'') and greedy decoding for at most 512 tokens. The frozen base-model generations are reused across comparisons with adapters trained on different topics.
We evaluate generations using \texttt{gpt-5.6-luna} \citep{openai2026gpt56luna} with medium reasoning effort and a two-stage protocol. The first stage decomposes each biography into distinct, atomic, and verifiable claims.
The second stage judges each claim and assigns one of three verdicts: \texttt{correct}, \texttt{incorrect}, or \texttt{uncertain}.
Factual precision is the mean, across biographies containing at least one claim, of the proportion of extracted claims judged correct.
We interpret this measure as LLM-estimated factual precision, which corresponds strongly but imperfectly with human judgment (\Cref{tab:llm_judge}).

\section{Supplementary Experimental Results}
\label{sec:supp_expt}

\paragraph{Held-out Scaling.}
In \Cref{sec:expt}, for non-probabilistic baseline scores not in the $[0,1]$ range, \textit{i.e.} Energy, Entropy, LN-Entropy, NLL, and Perplexity, we apply $\operatorname{min-max}$ linear scaling similar to \citep{vashurin-etal-2025-benchmarking}.
Differently from them, however, we scale using the entire set, whereas \citet{vashurin-etal-2025-benchmarking} scale using minimum and maximum values on a held-out calibration set and apply clipping.
We examine whether estimating the scaling parameters on the complete evaluation set affects our results through a robustness experiment over the English datasets. Using five label-stratified folds, we treat each fold as a 20\% calibration set, estimate the minimum and maximum of each fold, and scale the remaining 80\% with these statistics. We take the mean over the five folds.
\Cref{tab:scaling} presents the ECE results against the main \Cref{tab:main-results-language} values (AUROC remains unchanged).
Overall, held-out scaling does not change our substantive conclusions.
Thus, we use full-set scaling in the main experiments as the dataset is limited in size. This ensures that scaled methods are compared on the same examples as probabilistic methods that require no separate calibration set.

\begin{table}[htbp]
    \centering
    \setlength{\tabcolsep}{15pt}
    \renewcommand{\arraystretch}{1.2}

    \resizebox{.9\linewidth}{!}{%
    \begin{tabular}{lc}
    \toprule
    \textbf{Method} & \textbf{No Held-out $\rightarrow$ Held-out} \\ \midrule
    Energy & .1726 $\rightarrow$ .1735 \\
    Entropy &  .3065 $\rightarrow$ .2842 \\
    LN-Entropy & .1257 $\rightarrow$ .1357 \\
    NLL & .2937 $\rightarrow$ .2666 \\
    Perplexity & .4781 $\rightarrow$ .4689 \\ \midrule
    Average & .2753 $\rightarrow$ .2658\\
    \bottomrule
    \end{tabular}
    }
    \caption{ECE with $\operatorname{min-max}$ scaling using the entire dataset without a held-out calibration set or with one.}
    \label{tab:scaling}
\end{table}

\begin{table}[htbp]
    \centering
    \setlength{\tabcolsep}{3pt}
    \renewcommand{\arraystretch}{1.25}

    \resizebox{\linewidth}{!}{%
    \begin{tabular}{lccccc}
        \toprule
        \textbf{Method} & \textbf{IT}& \textbf{Chat} & \textbf{ECE ($\downarrow$)} & \textbf{BS ($\downarrow$)} & \textbf{AUROC ($\uparrow$)}\\
        \midrule
        \multirow{4}{*}{P(True)} & \cross & \cross & \mse{.2026}{.0050} & \mse{.2816}{.0014} & \mse{.6548}{.0144} \\
        & \cross & \checkmk & \mse{.4530}{.0014} & \mse{.4543}{.0011} & \mse{.6081}{.0158} \\
        & \checkmk & \cross & \mse{.3167}{.0041} & \mse{.3589}{.0023} & \mse{.5760}{.0154} \\
        & \checkmk & \checkmk & \mse{.4987}{.0008} & \mse{.4965}{.0008} & \mse{.5420}{.0163} \\ \midrule
        \multirow{4}{*}{P(True)-norm} & \cross & \cross & \mse{.1824}{.0076} & \mse{.2399}{.0004} & \mse{.7823}{.0113} \\
        & \cross & \checkmk & \mse{.2294}{.0043} & \mse{.2889}{.0019} & \mse{.6501}{.0148} \\
        & \checkmk & \cross & \mse{.1973}{.0076} & \mse{.2462}{.0006} & \mse{.7507}{.0123} \\
        & \checkmk & \checkmk & \mse{.2393}{.0043} & \mse{.3202}{.0011} & \mse{.5497}{.0167} \\ \midrule
        \multirow{4}{*}{Gen-Z} & \cross & \cross & \mse{.3656}{.0118} & \mse{.3656}{.0118} & \mse{.6347}{.0118} \\
        & \cross & \checkmk & \mse{.3984}{.0119} & \mse{.3984}{.0119} & \mse{.6022}{.0119} \\
        & \checkmk & \cross & \mse{.4016}{.0123} & \mse{.4016}{.0123} & \mse{.5993}{.0123} \\
        & \checkmk & \checkmk & \mse{.4214}{.0127} & \mse{.4214}{.0127} & \mse{.5798}{.0127} \\  \midrule
        \multirow{4}{*}{CFP (Ours)} & \cross & \cross & \mse{.1403}{.0100} & \mse{.2405}{.0005} & \mse{.7363}{.0129} \\
        & \cross & \checkmk & \mse{.1058}{.0107} & \mse{.2446}{.0004} & \mse{.6816}{.0150} \\
        & \checkmk & \cross & \mse{.1169}{.0107} & \mse{.2402}{.0009} & \mse{.6781}{.0146} \\
        & \checkmk & \checkmk & \mse{.0877}{.0103} & \mse{.2409}{.0009} & \mse{.6518}{.0156} \\
        \bottomrule
    \end{tabular}%
    }
    \caption{Results with different model types and prompt templates: ``IT'' \cross \ and \checkmk \ refer to base model and instruction-tuned model respectively, and ``Chat'' indicates usage of chat template.}
    \label{tab:template_posttrain}
\end{table}

\paragraph{English versus Native Prompting.}
\Cref{tab:native} examines whether expressing the confidence prompts in the target language affects performance on non-English datasets. The effects are method-dependent. For CFP, native prompting modestly improves discrimination for both base and IT models, while slightly increasing ECE. P(True)-norm instead shows small decreases in AUROC, with mixed effects on calibration.
Overall, prompt language can shift calibration and discrimination differently.

\begin{table}[htbp]
    \centering
    \setlength{\tabcolsep}{3pt}
    \renewcommand{\arraystretch}{1.25}

    \resizebox{\linewidth}{!}{%
    \begin{tabular}{lccccc}
        \toprule
        \textbf{Method} & \textbf{IT}& \textbf{Native} & \textbf{ECE ($\downarrow$)} & \textbf{BS ($\downarrow$)} & \textbf{AUROC ($\uparrow$)}\\
        \midrule
        \multirow{4}{*}{P(True)} & \cross & \cross & \mse{.1991}{.0042} & \mse{.2866}{.0013} & \mse{.6399}{.0148} \\
        & \cross & \checkmk & \mse{.1974}{.0056} & \mse{.2734}{.0015} & \mse{.6600}{.0142} \\
        & \checkmk & \cross & \mse{.3326}{.0038} & \mse{.3664}{.0025} & \mse{.5789}{.0156} \\
        & \checkmk & \checkmk & \mse{.3006}{.0041} & \mse{.3522}{.0119} & \mse{.5614}{.0154} \\ \midrule
        \multirow{4}{*}{P(True)-norm} & \cross & \cross & \mse{.1745}{.0077} & \mse{.2411}{.0003} & \mse{.7806}{.0113} \\
        & \cross & \checkmk & \mse{.1769}{.0076} & \mse{.2296}{.0004} & \mse{.7631}{.0019} \\
        & \checkmk & \cross & \mse{.1989}{.0076} & \mse{.2497}{.0006} & \mse{.7458}{.0125} \\
        & \checkmk & \checkmk & \mse{.1803}{.0078} & \mse{.2416}{.0005} & \mse{.7399}{.0127} \\ \midrule
        \multirow{4}{*}{Gen-Z} & \cross & \cross & \mse{.3829}{.0119} & \mse{.3829}{.0119} & \mse{.6174}{.0119} \\
        & \cross & \checkmk & \mse{.3632}{.0119} & \mse{.3632}{.0119} & \mse{.6370}{.0118} \\
        & \checkmk & \cross & \mse{.4154}{.0125} & \mse{.4154}{.0125} & \mse{.5862}{.0125} \\
        & \checkmk & \checkmk & \mse{.4047}{.0128} & \mse{.4047}{.0128} & \mse{.5953}{.0127} \\  \midrule
        \multirow{4}{*}{CFP (Ours)} & \cross & \cross & \mse{.1262}{.0103} & \mse{.2414}{.0006} & \mse{.7148}{.0138} \\
        & \cross & \checkmk & \mse{.1450}{.0101} & \mse{.2421}{.0004} & \mse{.7325}{.0129} \\
        & \checkmk & \cross & \mse{.1031}{.0108} & \mse{.2418}{.0009} & \mse{.6569}{.0156} \\
        & \checkmk & \checkmk & \mse{.1131}{.0107} & \mse{.2419}{.0008} & \mse{.6702}{.0146} \\
        \bottomrule
    \end{tabular}%
    }
    \caption{Results with different model types and English versus native prompting in the target language on non-English datasets.}
    \label{tab:native}
\end{table}

\paragraph{Performance Across Model Scale.}
\Cref{fig:model_size} disaggregates the overall results in \Cref{tab:main-results-language} by model family and size.
CFP achieves the lowest ECE at every evaluated scale, indicating that its aggregate calibration advantage is not driven by a single model family or size. Its AUROC generally improves with scale for Qwen-3 and Qwen-3.5 and remains competitive with P(True)-norm, although the latter is stronger at some individual scales.
The trend is not universal: Gemma-4 is non-monotonic, with E4B model outperforming 12B in AUROC.
Overall, model scale often improves factual discrimination, but does not guarantee it, whereas CFP's calibration advantage is comparatively consistent across scales; although, this must be interpreted cautiously as AUROC may fall close to chance despite low ECE.

\begin{figure*}
    \centering
    \includegraphics[width=\linewidth]{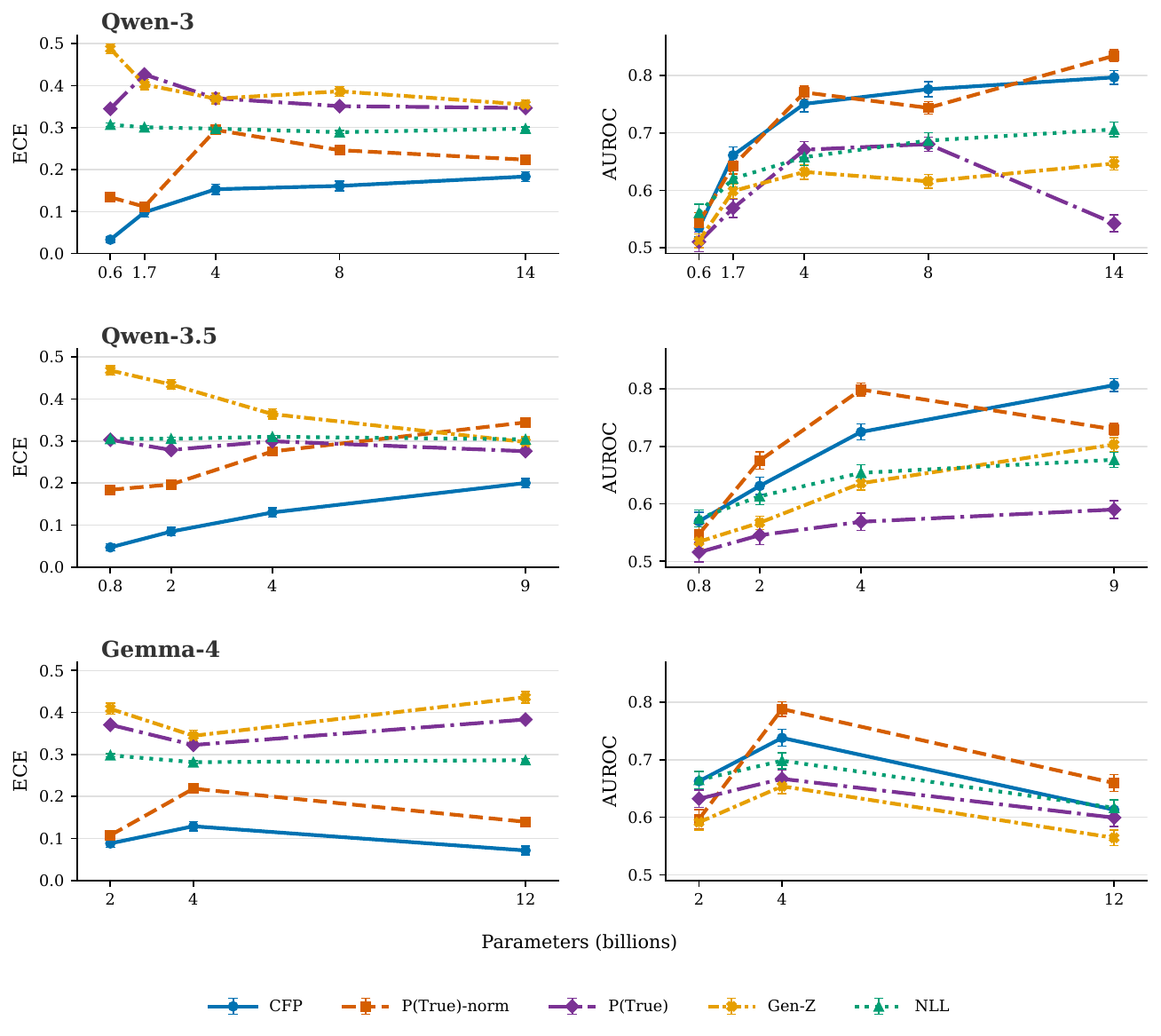}
    \caption{Calibration and  discrimination across model families and scales. Columns report ECE ($\downarrow$) and AUROC ($\uparrow)$.
    Each point is averaged over eight languages, base and IT model variants, and model/prompt-format settings (including English and native prompting but excluding base models with chat template).
    }
    \label{fig:model_size}
\end{figure*}

\paragraph{LLM-Judge Human Validation.}
To validate the \texttt{gpt-5.6-luna} factual labels in the biography generation task, we randomly select eight biographies each from English, German, and Hindi generations and ask two human annotators, with web access, to label the correctness of each claim.
Since we rely on internal authors as annotators due to budget constraints and they do not speak German and Hinid, we translate these claims into English using \texttt{gpt-5.6-sol}.
\Cref{tab:llm_judge} presents the validation results. As we have two sets of human annotations, we also compare agreement between the two annotators. Even between humans, there is no complete agreement on factuality labels.
The LLM-generated labels show good, although imperfect, agreement with human labels.
This validation serves as small-scale evidence of alignment with human judgment.

\begin{table}[htbp]
    \centering
    \setlength{\tabcolsep}{3pt}
    \renewcommand{\arraystretch}{1.25}

    \resizebox{\linewidth}{!}{%
    \begin{tabular}{lccc}
    \toprule
    \textbf{Language} & \textbf{\# Claims} & \textbf{Human---Human} & \textbf{LLM-Human} \\ \midrule
    English & 109 & 88.1\% (0.734) & 75.7\% (0.527) \\
    German & 71 & 88.7\% (0.761) & 85.9\% (0.719) \\
    Hindi & 71 & 100.0\% (1.000) & 83.1\% (0.634) \\
    \bottomrule
    \end{tabular}
    }
    \caption{Alignment between human and LLM judgment: agreement \% (Cohen's $\kappa$).}
    \label{tab:llm_judge}
\end{table}

\begin{table*}[htbp]
    \centering
    \setlength{\tabcolsep}{8pt}
    \renewcommand{\arraystretch}{1.15}
    \resizebox{\linewidth}{!}{%
    \begin{tabular}{llcccccc}
        \toprule
        \multirow{2}{*}{\textbf{Training Topic}} & \multirow{2}{*}{\textbf{Objective}} & \multicolumn{3}{c}{\textbf{German}} & \multicolumn{3}{c}{\textbf{Hindi}} \\
        \cmidrule(lr){3-5} \cmidrule(lr){6-8}
        & & \textbf{True--False} & \textbf{BB Target} & \textbf{BB English} & \textbf{True--False} & \textbf{BB Target} & \textbf{BB English} \\
        \midrule
        \multirow{3}{*}{Animals} & Base & 63.8 & 81.9 & 86.9 & 50.0 & 42.8 & 86.9 \\
         & Full & 56.9 \footnotesize \textcolor{darkred}{$(-6.9)$} & 81.9 {\footnotesize $(0.0)$} & 86.6 \footnotesize \textcolor{darkred}{$(-0.3)$} & 43.9 \footnotesize \textcolor{darkred}{$(-6.0)$} & 40.7 \footnotesize \textcolor{darkred}{$(-2.1)$} & 87.0 \footnotesize \textcolor{darkgreen}{$(+0.1)$} \\
         & CFE & 70.7 \footnotesize \textcolor{darkgreen}{$(+6.8)$} & 81.9 {\footnotesize $(0.0)$} & 87.1 \footnotesize \textcolor{darkgreen}{$(+0.2)$} & 59.9 \footnotesize \textcolor{darkgreen}{$(+9.9)$} & 40.7 \footnotesize \textcolor{darkred}{$(-2.1)$} & 86.7 \footnotesize \textcolor{darkred}{$(-0.2)$} \\
        \midrule
        \multirow{3}{*}{Chemical Elements} & Base & 81.1 & 81.9 & 86.9 & 54.5 & 42.8 & 86.9 \\
         & Full & 78.4 \footnotesize \textcolor{darkred}{$(-2.7)$} & 81.3 \footnotesize \textcolor{darkred}{$(-0.6)$} & 86.7 \footnotesize \textcolor{darkred}{$(-0.2)$} & 54.5 \footnotesize \textcolor{darkred}{$(-0.1)$} & 40.2 \footnotesize \textcolor{darkred}{$(-2.6)$} & 86.6 \footnotesize \textcolor{darkred}{$(-0.3)$} \\
         & CFE & 80.6 \footnotesize \textcolor{darkred}{$(-0.5)$} & 82.1 \footnotesize \textcolor{darkgreen}{$(+0.2)$} & 87.0 \footnotesize \textcolor{darkgreen}{$(+0.1)$} & 68.8 \footnotesize \textcolor{darkgreen}{$(+14.3)$} & 40.9 \footnotesize \textcolor{darkred}{$(-1.9)$} & 87.0 \footnotesize \textcolor{darkgreen}{$(+0.1)$} \\
        \midrule
        \multirow{3}{*}{Cities} & Base & 84.1 & 81.9 & 86.9 & 51.3 & 42.8 & 86.9 \\
         & Full & 60.9 \footnotesize \textcolor{darkred}{$(-23.2)$} & 81.2 \footnotesize \textcolor{darkred}{$(-0.7)$} & 87.2 \footnotesize \textcolor{darkgreen}{$(+0.3)$} & 48.9 \footnotesize \textcolor{darkred}{$(-2.4)$} & 38.1 \footnotesize \textcolor{darkred}{$(-4.7)$} & 86.7 \footnotesize \textcolor{darkred}{$(-0.2)$} \\
         & CFE & 87.8 \footnotesize \textcolor{darkgreen}{$(+3.7)$} & 81.8 \footnotesize \textcolor{darkred}{$(-0.1)$} & 87.4 \footnotesize \textcolor{darkgreen}{$(+0.6)$} & 51.1 \footnotesize \textcolor{darkred}{$(-0.2)$} & 42.9 \footnotesize \textcolor{darkgreen}{$(+0.1)$} & 87.0 \footnotesize \textcolor{darkgreen}{$(+0.1)$} \\
        \midrule
        \multirow{3}{*}{Companies} & Base & 74.7 & 81.9 & 86.9 & 54.7 & 42.8 & 86.9 \\
         & Full & 68.0 \footnotesize \textcolor{darkred}{$(-6.7)$} & 81.9 {\footnotesize $(0.0)$} & 87.3 \footnotesize \textcolor{darkgreen}{$(+0.4)$} & 51.8 \footnotesize \textcolor{darkred}{$(-2.8)$} & 39.1 \footnotesize \textcolor{darkred}{$(-3.7)$} & 87.4 \footnotesize \textcolor{darkgreen}{$(+0.6)$} \\
         & CFE & 84.4 \footnotesize \textcolor{darkgreen}{$(+9.7)$} & 81.9 {\footnotesize $(0.0)$} & 86.9 {\footnotesize $(0.0)$} & 56.2 \footnotesize \textcolor{darkgreen}{$(+1.5)$} & 41.4 \footnotesize \textcolor{darkred}{$(-1.3)$} & 86.4 \footnotesize \textcolor{darkred}{$(-0.4)$} \\
        \midrule
        \multirow{3}{*}{Inventions} & Base & 74.3 & 81.9 & 86.9 & 50.9 & 42.8 & 86.9 \\
         & Full & 71.1 \footnotesize \textcolor{darkred}{$(-3.2)$} & 81.9 {\footnotesize $(0.0)$} & 86.8 \footnotesize \textcolor{darkred}{$(-0.1)$} & 50.8 {\footnotesize $(0.0)$} & 42.6 \footnotesize \textcolor{darkred}{$(-0.2)$} & 86.7 \footnotesize \textcolor{darkred}{$(-0.2)$} \\
         & CFE & 78.4 \footnotesize \textcolor{darkgreen}{$(+4.1)$} & 82.1 \footnotesize \textcolor{darkgreen}{$(+0.2)$} & 86.8 \footnotesize \textcolor{darkred}{$(-0.1)$} & 51.0 \footnotesize \textcolor{darkgreen}{$(+0.1)$} & 41.8 \footnotesize \textcolor{darkred}{$(-1.0)$} & 86.9 {\footnotesize $(0.0)$} \\
        \midrule
        \multirow{3}{*}{Scientific Facts} & Base & 86.4 & 81.9 & 86.9 & 46.9 & 42.8 & 86.9 \\
         & Full & 84.0 \footnotesize \textcolor{darkred}{$(-2.4)$} & 82.1 \footnotesize \textcolor{darkgreen}{$(+0.2)$} & 86.8 \footnotesize \textcolor{darkred}{$(-0.1)$} & 53.7 \footnotesize \textcolor{darkgreen}{$(+6.9)$} & 40.7 \footnotesize \textcolor{darkred}{$(-2.1)$} & 86.9 {\footnotesize $(0.0)$} \\
         & CFE & 90.9 \footnotesize \textcolor{darkgreen}{$(+4.6)$} & 82.1 \footnotesize \textcolor{darkgreen}{$(+0.2)$} & 86.9 {\footnotesize $(0.0)$} & 55.4 \footnotesize \textcolor{darkgreen}{$(+8.5)$} & 42.3 \footnotesize \textcolor{darkred}{$(-0.4)$} & 86.7 \footnotesize \textcolor{darkred}{$(-0.2)$} \\
        \bottomrule
    \end{tabular}%
    }
    \caption{Per-topic alignment results for Qwen-3.5-2B. Each topic denotes the True--False training split used to fit a separate adapter. True--False reports held-out AUROC, and Belebele (BB) reports accuracy.}
    \label{tab:align_breakdown2b}
\end{table*}

\begin{table*}[htbp]
    \centering
    \setlength{\tabcolsep}{8pt}
    \renewcommand{\arraystretch}{1.15}
    \resizebox{\linewidth}{!}{%
    \begin{tabular}{llcccccc}
        \toprule
        \multirow{2}{*}{\textbf{Training Topic}} & \multirow{2}{*}{\textbf{Objective}} & \multicolumn{3}{c}{\textbf{German}} & \multicolumn{3}{c}{\textbf{Hindi}} \\
        \cmidrule(lr){3-5} \cmidrule(lr){6-8}
        & & \textbf{True--False} & \textbf{BB Target} & \textbf{BB English} & \textbf{True--False} & \textbf{BB Target} & \textbf{BB English} \\
        \midrule
        \multirow{3}{*}{Animals} & Base & 71.7 & 92.4 & 92.9 & 45.4 & 65.3 & 92.9 \\
         & Full & 46.6 \footnotesize \textcolor{darkred}{$(-25.2)$} & 91.9 \footnotesize \textcolor{darkred}{$(-0.6)$} & 92.9 {\footnotesize $(0.0)$} & 50.6 \footnotesize \textcolor{darkgreen}{$(+5.3)$} & 66.0 \footnotesize \textcolor{darkgreen}{$(+0.7)$} & 93.2 \footnotesize \textcolor{darkgreen}{$(+0.3)$} \\
         & CFE & 70.0 \footnotesize \textcolor{darkred}{$(-1.8)$} & 92.2 \footnotesize \textcolor{darkred}{$(-0.2)$} & 93.2 \footnotesize \textcolor{darkgreen}{$(+0.3)$} & 61.2 \footnotesize \textcolor{darkgreen}{$(+15.9)$} & 65.8 \footnotesize \textcolor{darkgreen}{$(+0.4)$} & 93.0 \footnotesize \textcolor{darkgreen}{$(+0.1)$} \\
        \midrule
        \multirow{3}{*}{Chemical Elements} & Base & 93.5 & 92.4 & 92.9 & 60.5 & 65.3 & 92.9 \\
         & Full & 91.5 \footnotesize \textcolor{darkred}{$(-2.0)$} & 92.0 \footnotesize \textcolor{darkred}{$(-0.4)$} & 93.1 \footnotesize \textcolor{darkgreen}{$(+0.2)$} & 63.6 \footnotesize \textcolor{darkgreen}{$(+3.1)$} & 64.3 \footnotesize \textcolor{darkred}{$(-1.0)$} & 93.0 \footnotesize \textcolor{darkgreen}{$(+0.1)$} \\
         & CFE & 93.4 \footnotesize \textcolor{darkred}{$(-0.1)$} & 92.4 {\footnotesize $(0.0)$} & 93.0 \footnotesize \textcolor{darkgreen}{$(+0.1)$} & 85.1 \footnotesize \textcolor{darkgreen}{$(+24.6)$} & 65.8 \footnotesize \textcolor{darkgreen}{$(+0.4)$} & 93.0 \footnotesize \textcolor{darkgreen}{$(+0.1)$} \\
        \midrule
        \multirow{3}{*}{Cities} & Base & 87.3 & 92.4 & 92.9 & 62.6 & 65.3 & 92.9 \\
         & Full & 77.9 \footnotesize \textcolor{darkred}{$(-9.4)$} & 92.2 \footnotesize \textcolor{darkred}{$(-0.2)$} & 92.9 {\footnotesize $(0.0)$} & 55.1 \footnotesize \textcolor{darkred}{$(-7.5)$} & 63.8 \footnotesize \textcolor{darkred}{$(-1.6)$} & 93.0 \footnotesize \textcolor{darkgreen}{$(+0.1)$} \\
         & CFE & 84.4 \footnotesize \textcolor{darkred}{$(-2.9)$} & 92.4 {\footnotesize $(0.0)$} & 93.1 \footnotesize \textcolor{darkgreen}{$(+0.2)$} & 65.6 \footnotesize \textcolor{darkgreen}{$(+2.9)$} & 65.1 \footnotesize \textcolor{darkred}{$(-0.2)$} & 93.0 \footnotesize \textcolor{darkgreen}{$(+0.1)$} \\
        \midrule
        \multirow{3}{*}{Companies} & Base & 80.8 & 92.4 & 92.9 & 57.7 & 65.3 & 92.9 \\
         & Full & 68.5 \footnotesize \textcolor{darkred}{$(-12.4)$} & 91.4 \footnotesize \textcolor{darkred}{$(-1.0)$} & 93.1 \footnotesize \textcolor{darkgreen}{$(+0.2)$} & 50.2 \footnotesize \textcolor{darkred}{$(-7.6)$} & 65.1 \footnotesize \textcolor{darkred}{$(-0.2)$} & 92.9 {\footnotesize $(0.0)$} \\
         & CFE & 87.4 \footnotesize \textcolor{darkgreen}{$(+6.6)$} & 92.4 {\footnotesize $(0.0)$} & 93.0 \footnotesize \textcolor{darkgreen}{$(+0.1)$} & 65.3 \footnotesize \textcolor{darkgreen}{$(+7.6)$} & 66.6 \footnotesize \textcolor{darkgreen}{$(+1.2)$} & 92.9 {\footnotesize $(0.0)$} \\
        \midrule
        \multirow{3}{*}{Inventions} & Base & 81.5 & 92.4 & 92.9 & 52.9 & 65.3 & 92.9 \\
         & Full & 83.5 \footnotesize \textcolor{darkgreen}{$(+2.0)$} & 92.4 {\footnotesize $(0.0)$} & 93.1 \footnotesize \textcolor{darkgreen}{$(+0.2)$} & 54.5 \footnotesize \textcolor{darkgreen}{$(+1.6)$} & 64.6 \footnotesize \textcolor{darkred}{$(-0.8)$} & 93.0 \footnotesize \textcolor{darkgreen}{$(+0.1)$} \\
         & CFE & 92.9 \footnotesize \textcolor{darkgreen}{$(+11.4)$} & 92.6 \footnotesize \textcolor{darkgreen}{$(+0.1)$} & 92.9 {\footnotesize $(0.0)$} & 65.2 \footnotesize \textcolor{darkgreen}{$(+12.3)$} & 66.1 \footnotesize \textcolor{darkgreen}{$(+0.8)$} & 93.0 \footnotesize \textcolor{darkgreen}{$(+0.1)$} \\
        \midrule
        \multirow{3}{*}{Scientific Facts} & Base & 94.9 & 92.4 & 92.9 & 62.7 & 65.3 & 92.9 \\
         & Full & 87.7 \footnotesize \textcolor{darkred}{$(-7.2)$} & 92.7 \footnotesize \textcolor{darkgreen}{$(+0.2)$} & 93.0 \footnotesize \textcolor{darkgreen}{$(+0.1)$} & 68.1 \footnotesize \textcolor{darkgreen}{$(+5.4)$} & 64.7 \footnotesize \textcolor{darkred}{$(-0.7)$} & 93.0 \footnotesize \textcolor{darkgreen}{$(+0.1)$} \\
         & CFE & 96.1 \footnotesize \textcolor{darkgreen}{$(+1.2)$} & 92.4 {\footnotesize $(0.0)$} & 93.0 \footnotesize \textcolor{darkgreen}{$(+0.1)$} & 80.5 \footnotesize \textcolor{darkgreen}{$(+17.8)$} & 66.2 \footnotesize \textcolor{darkgreen}{$(+0.9)$} & 93.0 \footnotesize \textcolor{darkgreen}{$(+0.1)$} \\
        \bottomrule
    \end{tabular}%
    }
    \caption{Per-topic alignment results for Qwen-3.5-4B. Each topic denotes the True--False training split used to fit a separate adapter. True--False reports held-out AUROC, and Belebele (BB) reports accuracy.}
    \label{tab:align_breakdown4b}
\end{table*}

\begin{table*}[htbp]
    \centering
    \setlength{\tabcolsep}{4pt}
    \renewcommand{\arraystretch}{1.5}
    \resizebox{\linewidth}{!}{%
    \begin{tabular}{llcccccc}
        \toprule
        \multirow{2}{*}{\textbf{Training Topic}} & \multirow{2}{*}{\textbf{Objective}} & \multicolumn{3}{c}{\textbf{English Generation}} & \multicolumn{3}{c}{\textbf{Hindi Generation}} \\
        \cmidrule(lr){3-5} \cmidrule(lr){6-8}
        & & \textbf{Precision (\%)} & \textbf{\# Claims} & \textbf{\# Correct} & \textbf{Precision (\%)} & \textbf{\# Claims} & \textbf{\# Correct} \\
        \midrule
        \multirow{3}{*}{Animals} & Base & 41.29 & 12.74 & 5.26 & 27.47 & 11.72 & 3.22 \\
        & Full & 40.24 \footnotesize \textcolor{darkred}{$(-1.05)$} & 11.68 \footnotesize \textcolor{darkred}{$(-1.06)$} & 4.70 \footnotesize \textcolor{darkred}{$(-0.56)$} & 33.57 \footnotesize \textcolor{darkgreen}{$(+6.10)$} & 8.34 \footnotesize \textcolor{darkred}{$(-3.38)$} & 2.80 \footnotesize \textcolor{darkred}{$(-0.42)$} \\
         & CFE & 42.17 \footnotesize \textcolor{darkgreen}{$(+0.89)$} & 12.14 \footnotesize \textcolor{darkred}{$(-0.60)$} & 5.12 \footnotesize \textcolor{darkred}{$(-0.14)$} & 37.04 \footnotesize \textcolor{darkgreen}{$(+9.57)$} & 13.12 \footnotesize \textcolor{darkgreen}{$(+1.40)$} & 4.86 \footnotesize \textcolor{darkgreen}{$(+1.64)$} \\
        \midrule
        \multirow{2}{*}{Chemical Elements} & Full & 39.46 \footnotesize \textcolor{darkred}{$(-1.83)$} & 12.52 \footnotesize \textcolor{darkred}{$(-0.22)$} & 4.94 \footnotesize \textcolor{darkred}{$(-0.32)$} & 29.96 \footnotesize \textcolor{darkgreen}{$(+2.48)$} & 9.08 \footnotesize \textcolor{darkred}{$(-2.64)$} & 2.72 \footnotesize \textcolor{darkred}{$(-0.50)$} \\
         & CFE & 44.98 \footnotesize \textcolor{darkgreen}{$(+3.70)$} & 11.56 \footnotesize \textcolor{darkred}{$(-1.18)$} & 5.20 \footnotesize \textcolor{darkred}{$(-0.06)$} & 34.12 \footnotesize \textcolor{darkgreen}{$(+6.64)$} & 11.08 \footnotesize \textcolor{darkred}{$(-0.64)$} & 3.78 \footnotesize \textcolor{darkgreen}{$(+0.56)$} \\
        \midrule
        \multirow{2}{*}{Cities} & Full & 37.77 \footnotesize \textcolor{darkred}{$(-3.52)$} & 12.18 \footnotesize \textcolor{darkred}{$(-0.56)$} & 4.60 \footnotesize \textcolor{darkred}{$(-0.66)$} & 33.73 \footnotesize \textcolor{darkgreen}{$(+6.25)$} & 3.38 \footnotesize \textcolor{darkred}{$(-8.34)$} & 1.14 \footnotesize \textcolor{darkred}{$(-2.08)$} \\
         & CFE & 40.73 \footnotesize \textcolor{darkred}{$(-0.56)$} & 11.54 \footnotesize \textcolor{darkred}{$(-1.20)$} & 4.70 \footnotesize \textcolor{darkred}{$(-0.56)$} & 34.78 \footnotesize \textcolor{darkgreen}{$(+7.31)$} & 8.74 \footnotesize \textcolor{darkred}{$(-2.98)$} & 3.04 \footnotesize \textcolor{darkred}{$(-0.18)$} \\
        \midrule
        \multirow{2}{*}{Companies} & Full & 39.24 \footnotesize \textcolor{darkred}{$(-2.05)$} & 12.64 \footnotesize \textcolor{darkred}{$(-0.10)$} & 4.96 \footnotesize \textcolor{darkred}{$(-0.30)$} & 33.55 \footnotesize \textcolor{darkgreen}{$(+6.08)$} & 12.16 \footnotesize \textcolor{darkgreen}{$(+0.44)$} & 4.08 \footnotesize \textcolor{darkgreen}{$(+0.86)$} \\
         & CFE & 43.76 \footnotesize \textcolor{darkgreen}{$(+2.47)$} & 11.70 \footnotesize \textcolor{darkred}{$(-1.04)$} & 5.12 \footnotesize \textcolor{darkred}{$(-0.14)$} & 29.52 \footnotesize \textcolor{darkgreen}{$(+2.05)$} & 10.84 \footnotesize \textcolor{darkred}{$(-0.88)$} & 3.20 \footnotesize \textcolor{darkred}{$(-0.02)$} \\
        \midrule
        \multirow{2}{*}{Inventions} & Full & 40.87 \footnotesize \textcolor{darkred}{$(-0.42)$} & 14.78 \footnotesize \textcolor{darkgreen}{$(+2.04)$} & 6.04 \footnotesize \textcolor{darkgreen}{$(+0.78)$} & 35.33 \footnotesize \textcolor{darkgreen}{$(+7.86)$} & 8.66 \footnotesize \textcolor{darkred}{$(-3.06)$} & 3.06 \footnotesize \textcolor{darkred}{$(-0.16)$} \\
         & CFE & 42.12 \footnotesize \textcolor{darkgreen}{$(+0.84)$} & 11.68 \footnotesize \textcolor{darkred}{$(-1.06)$} & 4.92 \footnotesize \textcolor{darkred}{$(-0.34)$} & 30.98 \footnotesize \textcolor{darkgreen}{$(+3.51)$} & 10.20 \footnotesize \textcolor{darkred}{$(-1.52)$} & 3.16 \footnotesize \textcolor{darkred}{$(-0.06)$} \\
        \midrule
        \multirow{2}{*}{Scientific Facts} & Full & 35.92 \footnotesize \textcolor{darkred}{$(-5.37)$} & 13.42 \footnotesize \textcolor{darkgreen}{$(+0.68)$} & 4.82 \footnotesize \textcolor{darkred}{$(-0.44)$} & 28.64 \footnotesize \textcolor{darkgreen}{$(+1.17)$} & 8.10 \footnotesize \textcolor{darkred}{$(-3.62)$} & 2.32 \footnotesize \textcolor{darkred}{$(-0.90)$} \\
         & CFE & 40.11 \footnotesize \textcolor{darkred}{$(-1.18)$} & 11.42 \footnotesize \textcolor{darkred}{$(-1.32)$} & 4.58 \footnotesize \textcolor{darkred}{$(-0.68)$} & 30.33 \footnotesize \textcolor{darkgreen}{$(+2.85)$} & 11.54 \footnotesize \textcolor{darkred}{$(-0.18)$} & 3.50 \footnotesize \textcolor{darkgreen}{$(+0.28)$} \\
        \bottomrule
    \end{tabular}%
    }
    \caption{Per-topic biography-generation results for Qwen-3.5-4B finetuned on Hindi True--False statements. Each topic denotes the training split used to fit a separate adapter. Factual precision is pooled across biographies containing at least one claim; \# Claims and \# Correct are per-entity means.}
    \label{tab:align_bio_breakdown}
\end{table*}

\end{document}